\documentclass[11pt]{article}

\usepackage{mkolar_definitions}
\usepackage{color}
\usepackage{fullpage}
\usepackage{wrapfig}
\usepackage{algorithm}[0.1]
\usepackage{algorithmic}
\usepackage{multirow}
\usepackage{tabularx}
\usepackage{array}
\usepackage{natbib}
\usepackage{mathabx}
\usepackage{graphicx}
\usepackage{rotating}
\usepackage{multirow}
\usepackage{smile}

\usepackage[usenames,dvipsnames,svgnames,table]{xcolor}
\usepackage[colorlinks=true,
            linkcolor=blue,
            urlcolor=blue,
            citecolor=blue]{hyperref}

\numberwithin{equation}{section}
\numberwithin{theorem}{section}

\usepackage[capitalize,nameinlink]{cleveref}

\crefname{section}{Section}{Sections}
\crefname{subsection}{Section}{Sections}
\Crefname{section}{Section}{Sections}
\Crefname{subsection}{Section}{Sections}

\Crefname{figure}{Figure}{Figures}

\crefformat{equation}{\textup{#2(#1)#3}}
\crefrangeformat{equation}{\textup{#3(#1)#4--#5(#2)#6}}
\crefmultiformat{equation}{\textup{#2(#1)#3}}{ and \textup{#2(#1)#3}}
{, \textup{#2(#1)#3}}{, and \textup{#2(#1)#3}}
\crefrangemultiformat{equation}{\textup{#3(#1)#4--#5(#2)#6}}%
{ and \textup{#3(#1)#4--#5(#2)#6}}{, \textup{#3(#1)#4--#5(#2)#6}}{, and \textup{#3(#1)#4--#5(#2)#6}}

\Crefformat{equation}{#2Equation~\textup{(#1)}#3}
\Crefrangeformat{equation}{Equations~\textup{#3(#1)#4--#5(#2)#6}}
\Crefmultiformat{equation}{Equations~\textup{#2(#1)#3}}{ and \textup{#2(#1)#3}}
{, \textup{#2(#1)#3}}{, and \textup{#2(#1)#3}}
\Crefrangemultiformat{equation}{Equations~\textup{#3(#1)#4--#5(#2)#6}}%
{ and \textup{#3(#1)#4--#5(#2)#6}}{, \textup{#3(#1)#4--#5(#2)#6}}{, and \textup{#3(#1)#4--#5(#2)#6}}

\crefdefaultlabelformat{#2\textup{#1}#3}

\makeatletter
\let\old@float@makebox\float@makebox
\renewcommand{\float@makebox}[1]{%
  \color@vbox\normalcolor
    \old@float@makebox{#1}%
  \color@endbox}
\makeatother

\usepackage{xargs}
\usepackage[colorinlistoftodos,prependcaption,textsize=tiny]{todonotes}
\newcommandx{\unsure}[2][1=]{\todo[linecolor=red,backgroundcolor=red!25,bordercolor=red,#1]{#2}}
\newcommandx{\change}[2][1=]{\todo[linecolor=blue,backgroundcolor=blue!25,bordercolor=blue,#1]{#2}}
\newcommandx{\info}[2][1=]{\todo[linecolor=OliveGreen,backgroundcolor=OliveGreen!25,bordercolor=OliveGreen,#1]{#2}}
\newcommandx{\improvement}[2][1=]{\todo[linecolor=Plum,backgroundcolor=Plum!25,bordercolor=Plum,#1]{#2}}

\newcommand{\email}[1]{\protect\href{mailto:#1}{#1}}

\begin{document}

\title{Personalized Federated Learning with Multiple Known Clusters}

\author{Boxiang Lyu \thanks{Booth School of Business, The University of Chicago, Chicago, IL. (\email{blyu@chicagobooth.edu})}
\and Filip Hanzely \thanks{Toyota Technological Institute at Chicago, Chicago, IL. (\email{fhanzely@gmail.com})}
\and Mladen Kolar \thanks{Booth School of Business, The University of Chicago, Chicago, IL. (\email{Mladen.Kolar@chicagobooth.edu})}}
\date{First draft: April 25th, 2022}

\maketitle

\begin{abstract}
    We consider the problem of personalized federated learning when there are known cluster structures within users. An intuitive approach would be to regularize the parameters so that users in the same cluster share similar model weights. The distances between the clusters can then be regularized to reflect the similarity between different clusters of users. We develop an algorithm that allows each cluster to communicate independently and derive the convergence results. We study a hierarchical linear model to theoretically demonstrate that our approach outperforms agents learning independently and agents learning a single shared weight. Finally, we demonstrate the advantages of our approach using both simulated and real-world data.
\end{abstract}

 \section{Introduction}

Smart phones, voice assistants, and wearable devices are everywhere in our lives, constantly collecting data on our behavior and habits. Federated learning (FL) is a recently introduced framework developed to use this rich source of data, while minimizing the intrusion of clients' privacy. Although traditional machine learning methods often require the aggregation of client data in a central server, federated learning avoids such a requirement, allowing models to be trained with mostly local computation and occasional server-wide communication rounds \citep{mcmahan2017communication, yang2019federated, bonawitz2019towards, li2020federated, kairouz2019advances}.

Data sets used in FL tasks are heterogeneous in nature, as they are collected from clients who are heterogeneous in nature. Tailoring the learned model to each client through personalized FL has garnered a tremendous amount of interest in recent years \citep{hanzely2020lower, fallah2020personalized, deng2020adaptive, dinh2020personalized, mansour2020three}. A central theme is learning a single global model and a personalized model for each client simultaneously in the training process \citep{hanzely2020federated, dinh2020personalized, li2021ditto, fallah2020personalized}. While it is typical to assume that clients can be grouped into a single cluster in existing literature, in disciplines such as education, psychology, or economics, it is often assumed that clients can be grouped into multiple clusters using known information. Hierarchical linear models are frequently used to model data with such a cluster structure \cite{raudenbush1986hierarchical, raudenbush1988educational, bryk1987application, bryk1992hierarchical, hofmann1997overview, stephen2002hierarchical}. 

Motivated by the prevalence of hierarchical models in social sciences, we propose a hierarchical, multi-cluster approach to personalization in federated learning. In particular, we leverage the known hierarchical structure and simultaneously learn (1) a global model for all clients, (2) a cluster-specific model for each client cluster, and (3) a personalized model for each client. More specifically, we develop a loopless algorithm for fitting the hierarchical linear model and derive its convergence rates and optimal parameters. Our algorithm allows each client cluster to communicate only within the group and allows each cluster to determine when to aggregate local updates independently of other clusters. Such an independence could improve convergence when the client clusters align with the communication graph. Existing empirical research has shown that extra communication rounds within different geographical clusters can improve convergence rates \citep{huang2019patient,briggs2020federated} and the proposed algorithm can leverage this property. In the finite-sum setting, we further develop an accelerated, variance reduced, stochastic variant of the algorithm, which has been shown to be minimax optimal in terms of communication rounds and oracle calls in the single cluster setting \citep{hanzely2020lower}.

Some existing work studied federated learning under the assumption that clients can be grouped into different clusters without assuming that these clusters are known. Our work emphasizes the optimality of the proposed approach when clusters are given a priori. In particular, the prior literature mostly focuses on the convergence rates of the optimization procedure and identifies the conditions under which the underlying cluster structures can be recovered. A key problem remains unaddressed: even assuming the cluster structure is known, can we prove that these structures can improve upon simple baselines such as training using only local data or training a single model? More importantly, can we show the optimality of the clustered formulation even in the simple case where the cluster labels are given?

To answer these questions, we establish statistical properties for the estimator obtained in our framework. We show that under a simple problem setting, our approach recovers the best linear unbiased estimator of the clients' local weights, dominating both training a single model for all clients and training a unique model for each client independently. This result complements \cite{li2021ditto} that studied the single-group setting. We further show that the unbiased restriction cannot be weakened. Even in the single cluster setting, we can construct a counterexample where a James-Stein estimator outperforms the proposed method. Our result also complements \cite{chen2021theorem} by identifying a regime in which a personalized approach is at least as good as the recommended alternatives, noting that structural information can be used to develop personalization schemes that outperform the proposed alternatives.

Finally, we demonstrate the empirical effectiveness of our model on the DMEF Customer Lifetime Value data set \cite{blattberg2009customer} and compare it with a recently proposed method in targeted marketing \cite{bumbaca2020scalable}, developing a personalized marketing model using data from a leading non-profit organization in the United States~\cite{blattberg2009customer}. 

\subsection{Related work}

Our work is related to the literature on personalized federated learning, distributed multi-task learning, and Bayesian hierarchical models.

There has been a lot of focus on personalization in federated learning over the past few years. \cite{karimireddy2020scaffold} is one of the first works to discuss heterogeneity among clients. \cite{fallah2020personalized} adapted model-agnostic meta-learning (MAML) algorithms for personalization. \cite{deng2020adaptive} developed a method that interpolates between client-specific parameters and global parameters. \cite{mansour2020three} suggests three different approaches for personalized federated learning and provides learning theoretic guarantees. \cite{li2021ditto, hanzely2020federated, dinh2020personalized} studied a personalization approach where each client has a parameter whose distance to the average of the parameters is regularized either explicitly or implicitly and developed different optimization techniques. \cite{li2018federated} studied optimizing a similar regularized loss in a non-personalized setting, where a single model is trained for all clients. More specifically, \cite{li2021ditto} provided some theoretical justification on the robustness and fairness of the procedure in the single cluster regime, \cite{hanzely2021personalized} provided a unified analysis of different optimization techniques, and \cite{dinh2020personalized} studied the optimization problem assuming that individual clients can exactly evaluate a proximal operator. These approaches can be viewed as a special instance of distributed multi-task learning with graph regularization \cite{wang2018distributed}.

A different line of work on personalization assumes that client-specific parameters may be drawn from an unknown mixture distribution and simultaneously group clients and learn model parameters using a single algorithm \citep{mansour2020three, ghosh2020efficient, sattler2020clustered, smith2017federated, briggs2020federated, huang2019patient}. Similar approaches have been studied in multi-task learning  \citep{kumar2012learning, jacob2008clustered, zhang2017survey, zhang2012convex, zhou2011clustered,zhou2011malsar, bakker2003task}. While these approaches focus on the setting where the structure and parameters of the cluster are learned simultaneously, in many domains, the clusters are known a priori and given, for example, by known covariates such as age, gender, and geographical location \citep{raudenbush1986hierarchical, raudenbush1988educational, bryk1987application, bryk1992hierarchical, hofmann1997overview, stephen2002hierarchical, lee1996hierarchical, daniels1999hierarchical}. Such a structure can be used to develop hierarchical models with improved personalization without having to separately cluster the clients. The data example that we investigate in detail in~\cref{sec:experiments} comes from marketing, where hierarchical models have a long history \citep{naik2009hierarchical, bumbaca2017distributed, bumbaca2020scalable, hooley1999marketing, french2015hierarchical}. Given the ubiquity of mobile devices, combining these marketing models with federated learning could better help the industry implement state-of-the-art marketing research.


Our optimization procedure is related to the methods used to optimize the objectives commonly found in hierarchical federated learning \citep{abad2020hierarchical, wang2020local, liu2020client, briggs2020federated, wainakh2020enhancing}. However, while the existing literature focuses on finding a single model for all clients, our optimization procedure learns personalized models for each client and each cluster and also learns a joint global model. The design of the optimization algorithm is related to loopless procedures in distributed optimization \citep{zhao2021fedpage, li2021anita, qian2021error}. These procedures remove the inner loops, thereby simplifying the algorithm. In single-machine settings, such simplifications have been shown to outperform their loopy counterparts \citep{kovalev2020don}.

Concurrent to our work, \cite{marfoq2021federated} studies federated multi-task learning under a mixture of distributions, focusing on nonasymptotic convergence rates. Our work further shows the optimality of our approach in terms of generalization error in addition to convergence analysis. Additionally, instead of analyzing an expectation-maximization-inspired approach, our work uses a loopless gradient-based algorithm that has been shown to enjoy optimal communication complexity in the single cluster regime. Furthermore,~\cite{duan2022adaptive} discusses a general framework for multi-task learning and shows that the approach can be adapted to various concepts of task relatedness. Although the loss function discussed here is similar, we further develop a federated optimization procedure and characterize the bounds on the communication and computation complexity of the federated learning algorithm used to minimize the loss.
\subsection{Notation} 

For any vector $v \in \RR^d$, we use $\|v\|$ to denote its $\ell_2$ norm. 
For any finite set $A$, we use $|A|$ to denote its cardinality and $A[i]$ to denote the $i$-th element in $A$ according to some arbitrary order.
The $d$-dimensional identity matrix is denoted as $I_d \in \RR^{d \times d}$.
The Kronecker product between two conforming matrices $B, C$ is denoted as $B \otimes C$.

\section{Model Formulation}

Suppose that there are $n$ clients with their individual data sets for whom we would like to fit personalized models. Furthermore, suppose that these clients are divided into $k$ known clusters. Generally speaking, if two clients belong to the same cluster, then we expect their models to be more similar than if the clients belong to different clusters. Let $\cI_j$, $j = 1, \ldots, k$, be the set of clients belonging to the cluster $j$. We use $f_i(\cdot)$ to denote the loss function for client $i = 1, \ldots, n$. Throughout the paper, we assume that the loss function $f_i(\cdot)$ is strongly convex and smooth. In particular, we make the following assumption.
\begin{assumption}
    \label{assumption:smooth_and_convex}
    The loss function $f_i$ is $\mu$-strongly convex and $L$-smooth.
\end{assumption} 
We use ``local'', ``cluster'', and ``global`` to denote variables, functions, and values associated with individual clients, different clusters, and the entire network, respectively.

Let $\theta_i \in \RR^d$ denote the model parameter for client $i$. We focus on minimizing the following objective function:
\begin{equation}
    \label{eqn:loss}
    \min_{\{\theta_i\}_{i = 1}^n} F(\{\theta_i\}_{i = 1}^n) \coloneq \sum_{j = 1}^k\sum_{i \in \cI_j}\left( f_i(\theta_i) + \frac{(1 - \alpha_j)\gamma_i}{2}\nbr{\theta_i - \Bar{\theta}_j}^2 + \frac{\alpha_j\gamma_i}{2}\nbr{\theta_i - \Bar{\theta}}^2\right),
\end{equation}
where $\alpha_j \in \RR_{\geq 0}$, $j=1,\ldots,k$, are the tuning parameters that control the regularization strength in each cluster, $\gamma_i \in \RR_{\geq 0}$, $i=1,\ldots,n$, are the tuning parameters specific to each client, and $\{\Bar{\theta}_j\}_{j = 1}^n, \Bar{\theta}$ denote the weight averages of the parameters in the cluster $j$ and the entire network, respectively. That is, 
\begin{equation}
\label{eqn:theta_bar_defn}
\Bar{\theta}_j = \frac{\sum_{i \in \cI_j}\gamma_i \theta_i}{\sum_{i \in \cI_j}\gamma_i}, \quad j = 1, \ldots, k,
\quad \text{and} \quad
\Bar{\theta} = \frac{\sum_{j = 1}^k \alpha_j \Bar{\theta}_j}{\sum_{j = 1}^k \alpha_j} = \frac{\sum_{j = 1}^k \sum_{i \in \cI_j} \alpha_j\gamma_i \theta_i}{\sum_{j = 1}^k \sum_{i \in \cI_j}\alpha_j\gamma_i}. 
\end{equation}

The objective is comprised of three terms: the client-specific loss and two regularization terms, summed over all clients. The first regularizer penalizes the distance between the local parameter $\theta_i$ and the cluster averages $\Bar{\theta}_j$, while the second regularizer penalizes the distance between the local parameter and the global average $\Bar{\theta}$. By changing the cluster-specific parameter $\alpha_j$ from $0$ to $1$ we can interpolate between the two regimes: when $\alpha_j = 0$ for all $j$, we train $k$ personalized models---one for each cluster, independently of other clusters; when $\alpha_j = 1$ for all $j$,
we train the single-cluster model studied in \cite{hanzely2020federated, dinh2020personalized, li2021ditto}.

The objective in \cref{eqn:loss} is different from commonly used objective functions in multi-task learning \citep{wang2018distributed, zhou2011clustered, zhou2011malsar, jacob2008clustered}. When the cluster structure is known, the multi-task learning objective can be written as
\begin{equation}
    \label{eqn:mtl_loss}
    \begin{split}
        \min_{\{\theta_i\}_{i = 1}^n, \{w_j\}_{j = 1}^k, \Bar{w}} & F_{MTL}(\{\theta_i\}_{i = 1}^n, \{w_j\}_{j = 1}^k, \Bar{w}) \\
        &\quad\coloneq \sum_{j = 1}^k \left(\sum_{i \in \cI_j} \left(f_i(\theta_i) + \frac{\gamma_i}{2}\|\theta_i - w_j\|^2\right) + \frac{\lambda_j}{2}\|w_j - \Bar{w}\|^2\right),
    \end{split}
\end{equation}
where the penalty parameter $\gamma_i$ regularizes the distance between the local parameter and the cluster parameter $w_j$, while the penalty parameter $\lambda_j$, $j = 1, \ldots, k$, regularizes  the distance between the cluster parameter and the global parameter $\Bar{w}$. Furthermore, in \eqref{eqn:mtl_loss} we optimize both the local weights $\{\theta_i\}_{i = 1}^n$ and the average local parameters, $\{w_j\}_{j = 1}^k$ and $\Bar{w}$. The two forms, however, have same stationary points.

\begin{proposition}
    \label{prop:equiv_mtl}
    Suppose that~\cref{assumption:smooth_and_convex} holds for loss functions $\{f_i\}_{i = 1}^n$. Let $\{\lambda_j\}_{j = 1}^k, \{\gamma_i\}_{i = 1}^n$ be any set of tuning parameters for \cref{eqn:mtl_loss}. Let $\{\hat{w}_j\}_{j = 1}^k$, $\hat{\Bar{w}}$, and $\{\hat{\theta}_i'\}_{i = 1}^n$ be the minimizers of \eqref{eqn:mtl_loss}. Fix
    \begin{equation*}
        \alpha_j = \frac{\lambda_j}{\lambda_j + \sum_{i \in \cI_j} \gamma_i}, \quad j = 1, \ldots, k,
    \end{equation*}
    and let $\{\hat{\theta}_i\}_{i = 1}^n$ denote the unique minimizer of \cref{eqn:loss}. We then have $\hat{\theta}_i = \hat{\theta}_i'$ for all $i \in [n]$.
\end{proposition}
\begin{proof}
See~\cref{subsec:proof_of_prop_equiv_mtl}.
\end{proof}
As we will demonstrate in the sequel, while~\cref{eqn:loss,eqn:mtl_loss} have the same stationary points, optimizing the former does not require us to keep track of cluster or network averages, allowing these parameters to be calculated on the fly. This feature of~\cref{eqn:loss} better suits the federated learning setting, removing the server's need to keep track of additional variables during the optimization process.

To further motivate the objective in~\cref{eqn:loss}, we show that the maximum likelihood estimate of all parameters in a hierarchical generalized linear model can be viewed as the minimizer of~\cref{eqn:loss} \citep{lee1996hierarchical, stephen2002hierarchical, bryk1992hierarchical}. Specifically, let
\begin{align*}
    \Bar{\theta}_j^* \sim \cN\rbr{\Bar{\theta}^*, \sigma_{\text{center}}^2I_d}, \quad j = 1, \ldots, k,
\end{align*} 
where $\cN(\cdot, \cdot)$ denotes a Gaussian distribution and
$\Bar{\theta}^* \in \RR^d$. Furthermore, assume that
for $j = 1, \ldots, k$ and $i \in \cI_j$, we have
\begin{align*}
    \theta_i^* & \sim \cN\rbr{\Bar{\theta}^*_j, \sigma_{j}^2 I_d}, \\
    y_i \mid \theta_i^*, X_i & \sim p_Y(y_i; u^{-1}(X_i^T\theta_i^*), \tau),
\end{align*} 
where $X_i \in \RR^{n_i \times d}$ is a matrix of observations, $y_i \in \RR^{n_i}$ is the response vector, $u$ is a known link function, $p_Y(\cdot; \cdot, \cdot)$ is the probability density function of an exponential family distribution. The negative log-likelihood for this model is
\begin{equation*}
    \begin{split}
        &\ell(\{\theta_i\}_{i = 1}^n, \{\Bar{\theta}_j\}_{j = 1}^k ,\Bar{\theta}; \{X_i, y_i\}_{i = 1}^n, \tau) =\\ &\quad \sum_{j = 1}^k \Biggl(\sum_{i \in \cI_j}\rbr{\log p_Y(y_i; u^{-1}(X^T_i\theta_i^*), \tau) + \frac{1}{2\sigma_{\text{cluster }j}^2}\nbr{\theta_i - \Bar{\theta}_j}^2}\\
        &\hspace{20em} + \frac{1}{2\sigma_{\text{center}}^2}\nbr{\Bar{\theta}_j - \Bar{\theta}}^2\Biggr),
    \end{split}
\end{equation*}
which is a special instance of~\cref{eqn:mtl_loss}. Then, by~\cref{prop:equiv_mtl}, minimizing~\cref{eqn:loss} with suitable tuning parameters corresponds to the maximum likelihood estimation.

\section{Algorithm and Convergence Analysis}

We develop a loopless SGD-style algorithm to minimize the objective in~\cref{eqn:loss}. In each iteration, the algorithm randomly decides between descending on the loss functions $f_i$'s, descending on the regularizer controlling the distance between the local weights and the centers of the local clusters, and descending on both regularizers. The first kind of descent step can be done in parallel on all machines, the same as the local computation step in local SGD~\citep{stich2018local}. The second kind requires clients within the same cluster to communicate with one another, but does not require between-cluster communication. Only when we simultaneously perform gradient descent on both the within-cluster and between-cluster regularizers do we communicate across clusters.

\subsection{Technical Preliminaries}

We start by introducing an additional notation. Recall that $i = 1, \ldots, n$ is used to index clients, 
$j = 1, \dots, k$ is used to index clusters, 
$\cI_j$ is the set of clients belonging to the cluster $j$, 
with cardinality $|\cI_j|$. Let $\Theta_j \in \RR^{|\cI_j|d}$ be the weight vector formed by stacking the weights of the clients in the cluster $j$ and $\Theta \in \RR^{nd}$ be the weight vector formed by stacking the weights of all clients; that is,
\begin{equation}
\label{eqn:btheta_defn}
\btheta_j = (\theta_{\cI_j[1]}, \ldots, \theta_{\cI_j[|\cI_j|]}) \in \RR^{|\cI_j|d}
\quad\text{and}\quad
\btheta = (\theta_1, \ldots, \theta_n) \in \RR^{nd}.
\end{equation} 
Without loss of generality, assume that clients are ordered according to the cluster to which they belong, that is,
$\cI_1 = \{1, \ldots, |\cI_1|\}$ and $\cI_k = \{|\cI_{k-1}| + 1, \ldots, |\cI_{k-1}| + |\cI_{k}|\}$. We may then write the cluster-specific and global regularizers as
\begin{align*}
&\psi_j\left(\btheta_j; \cbr{\gamma_i}_{i \in \cI_j}\right) = \frac{1}{2}\sum_{i \in \cI_j}\gamma_i \nbr{\theta_i - \Bar{\theta}_j}^2, 
\qquad j=1,\ldots,k,\\
&\varphi\rbr{\btheta; \cbr{\gamma_i}_{i = 1}^n, \cbr{\alpha}_{j = 1}^k} = \frac{1}{2}\sum_{j = 1}^k \alpha_j \sum_{i \in \cI_j} \gamma_i \nbr{\theta_i - \Bar{\theta}}^2.
\end{align*} 
The loss function in~\cref{eqn:loss} can now be rewritten as
\[
    \min_{\btheta} F(\btheta) = \sum_{j = 1}^k \sum_{i \in \cI_j}f_i(\theta_i) + \sum_{j = 1}^k (1 - \alpha_j)\psi_j\left(\btheta_j; \cbr{\gamma_i}_{i \in \cI_j}\right) + \varphi\rbr{\btheta; \cbr{\gamma_i}_{i = 1}^n, \cbr{\alpha}_{j = 1}^k}.
\]
We explicitly calculate the gradients of the cluster-specific regularizers and the global regularizer in the following proposition.

\begin{proposition} \label{prop:reg_grad}
    We have that
    \begin{align*}
        \nabla_{\theta_i}\psi_j\rbr{\btheta_j; \cbr{\gamma_i}_{i \in \cI_j}} = \gamma_i\rbr{\theta_i - \Bar{\theta}_j}, 
        \qquad i \in \cI_j, j = 1,\ldots, k; \\
        \nabla_{\theta_i}\varphi\rbr{\btheta; \cbr{\gamma_i}_{i = 1}^n, \cbr{\alpha}_{j = 1}^k} = \alpha_j\gamma_i\rbr{\theta_i - \Bar{\theta}},
        \qquad i \in \cI_j, j = 1,\ldots, k.
    \end{align*}
\end{proposition}
\begin{proof}
    See~\cref{subsec:proof_of_prop_reg_grad}.
\end{proof}

\subsection{Asynchronous Loopless Local Gradient Descent (Async-L2GD)}

We formally introduce our main algorithm, Async-L2GD, in this section. \Cref{alg:async_l2gd} provides the pseudocode.
We present an accelerated, variance reduced, and stochastic variant of the algorithm for the finite-sum setting in~\Cref{app:alg_variants}.
\begin{algorithm}[htbp]
\caption{Async-L2GD}
\label{alg:async_l2gd}
\begin{algorithmic}[1]
    \STATE \textbf{Input:} $\theta_1^0 = \dots = \theta_n^0 = \bm{0}_d \in \RR^d$, step size $\eta > 0$, probabilities $p_1, \dots, p_k, p_0 \in [0, 1]$, fractions $\tau_1, \dots, \tau_k \in [0, 1]$.
    \FOR{$t = 1, 2, \dots$}
    \STATE $\xi_0 = 1$ with probability $p_0$ and 0 with probability $1 - p_0$.
    \IF {$\xi_0 = 1$}
        \STATE All \textcolor{blue}{Clusters} $j = 1, \dots, k$ compute cluster average $\Bar{\theta}_j^t = \frac{\sum_{i \in \cI_j}\gamma_i \theta_i^t}{\sum_{i \in \cI_j}\gamma_i}.$
        \STATE All \textcolor{blue}{Clusters} communicate with one another and calculate network average\\
        \quad $\Bar{\theta}^t = \frac{\sum_{j = 1}^k \sum_{i \in \cI_j} \alpha_j \gamma_i \theta_i^t}{\sum_{j = 1}^k \sum_{i \in \cI_j} \alpha_j \gamma_i}$.
        \STATE All \textcolor{blue}{Clusters} $j = 1, \dots, k$ compute step towards both cluster average and network average for all clients $i \in \cI_j$\\
        \quad$\theta_i^{t + 1} = \rbr{1 - \frac{\eta \gamma_i(\alpha_j + \tau_j(1 - \alpha_j))}{p_0}}\theta_i^t + \frac{\eta\gamma_i}{p_0}(\alpha_j\Bar{\theta}^t + \tau_j(1 - \alpha_j)\Bar{\theta}_j^t)$. 
    \ELSE
        \FOR{All \textcolor{blue}{Clusters} $j = 1, \dots, k$ in parallel}
            \STATE $\xi_j = 1$ with probability $p_j$ and 0 with probability $1 - p_j$.
            \IF{$\xi_j = 1$}
                \STATE Compute cluster average $\Bar{\theta}_j^t = \frac{\sum_{i \in \cI_j}\gamma_i \theta_i^t}{\sum_{i \in \cI_j}\gamma_i}.$
                \STATE Compute step towards cluster average for all clients $i \in \cI_j$ \\
                \quad $\theta_i^{t + 1} = \rbr{1 - \frac{\eta \gamma_i(1 - \tau_j)(1 - \alpha_j)}{(1 - p_0)p_j}}\theta_i^t + \frac{\eta \gamma_i(1 - \tau_j)(1 - \alpha_j)}{(1 - p_0)p_j}\Bar{\theta}_j^t$.
            \ELSE
                \STATE All \textcolor{red}{Clients} $i \in \cI_j$ perform a local gradient descent step \\
                \quad $\theta_i^{t + 1} = \theta_i^t - \frac{\eta}{(1 - p_0)(1 - p_j)}\nabla f_i(\theta_i^t)$.
            \ENDIF
        \ENDFOR
    \ENDIF
    \ENDFOR
\end{algorithmic}
\end{algorithm}

At the beginning of each round in \Cref{alg:async_l2gd}, the network randomly decides whether to aggregate the averages between all clusters or not. If a global aggregation round is performed, each cluster first calculates its cluster average and then communicates with one another (or a central server) to compute the between-cluster aggregate.

In a round that does not involve communication between clusters, we allow each cluster to independently decide whether to communicate within the cluster or not by randomly sampling $\xi_j \sim \text{Bernoulli}(p_j)$. By allowing $p_j$ to be different across clusters, we effectively allow different clusters to have different communication schedules that are specific to them. In real-world applications, this flexibility can be appreciated. For example, when we partition clients according to their geographical locations, we effectively allow clients in different regions to communicate according to different schedules, which could reduce communication latency and energy consumption \citep{abad2020hierarchical, liu2020client}. We characterize the impact of $\{p_j\}_{j = 0}^k$ on the number of communication rounds in \cref{prop:expected_num_comm}.

\begin{proposition}[Expected Number of Communication Rounds]
    \label{prop:expected_num_comm}
    Suppose that \Cref{alg:async_l2gd} is run for $T$ rounds.
    The expected number of communication rounds between clusters is $p_0(1 - p_0)T$. The expected number of communication rounds within the cluster $j$ is $(1 - p_0)p_j(1 - p_j)T$.
\end{proposition}
\begin{proof}
    By Lemma 4.3 in \cite{hanzely2020lower}, the expected number of between-cluster communication rounds is given by $p_0(1 - p_0)T$. A within-cluster communication round occurs only when $\xi_0 = 0$, which occurs with probability $1 - p_0$.
\end{proof}

As can be seen from \cref{alg:async_l2gd}, any communication between clusters necessarily implies that all clusters have computed within-cluster average. Intuitively speaking, it would then be more efficient to optimize $\psi_j\rbr{\btheta_j; \{\gamma_i\}_{i \in \cI_j}}$ whenever the cluster averages are computed in both the within-cluster and between-cluster aggregation rounds. Unfortunately, a naive implementation of this idea would lead to a biased gradient oracle. Suppose that we make a gradient descent on $\psi_j$ and $\varphi$ with the same learning rate. Since we descend on $\psi_j$ in both types of communication rounds, using the same learning rate for both $\psi_j$ and $\varphi$ would effectively cause us to descend twice on $\psi_j$, as the stochastic gradient oracle is biased towards updating $\psi_j$ more frequently. We introduce variables $\tau_1, \ldots, \tau_j$ to scale the effective step size for $\psi_j$ in the two different types of communication rounds, thereby ensuring that the stochastic gradient oracle is unbiased.

\cref{alg:async_l2gd} induces a stochastic gradient oracle that is equivalent to SGD with the following oracle. Assuming $p_j, \tau_j \in (0, 1)$, $j = 0, \ldots, k$,~\cref{alg:async_l2gd} defines a stochastic gradient oracle for $F(\btheta)$, denoted 
$$G(\btheta) = (G_1(\theta_1)^T, \ldots, G_n(\theta_n)^T)^T \in \RR^{nd},$$ 
where for each $j$ and $i \in \cI_j$, $G_i(\theta_i) \in \RR^d$ and
\begin{equation}
    \label{eqn:gradient_oracle}
    G_i(\theta_i) = \begin{cases}
        \frac{\gamma_i\alpha_j}{p_0}\rbr{\theta_i^t - \Bar{\theta}^t} + \frac{\gamma_i\tau_j(1 - \alpha_j)}{p_0}\rbr{\theta_i^t - \Bar{\theta_j^t}}, &\text{ if $\xi_0 = 1$}\\
        \frac{ \gamma_i(1 - \tau_j)(1 - \alpha_j)}{(1 - p_0)p_j}\rbr{\theta_i^t - \Bar{\theta}_j^t}, &\text{ if $\xi_0 = 0$ and $\xi_j = 1$}\\
        \frac{1}{(1 - p_0)(1 - p_j)}\nabla f_i(\theta_i^t), &\text{ if $\xi_0 = \xi_j = 0$}.
    \end{cases}
\end{equation}
Note that $G_i(\theta_i)$ is an unbiased estimator of $\nabla_{\theta_i}F(\btheta)$ for all $i$, ensuring that $G(\btheta)$ is an unbiased estimator of $\nabla_{\btheta}F(\btheta)$. Intuitively speaking, we can view \cref{alg:async_l2gd} as a stochastic gradient descent procedure on $F(\btheta)$ with step size $\eta$, using a noisy gradient oracle $G(\cdot)$.

The variance of $G(\btheta)$ with respect to $\{\xi_j\}_{j = 0}^k$ depends on a combination of $\{p_j\}_{j=0}^k$ and $\{\tau_j\}_{j = 1}^k$. Fixing $\tau_j$ to some arbitrary value is suboptimal, and we should properly tune the parameter. Intuitively speaking, when $p_j$ is small, within-cluster communications are less frequent. A larger $\tau_j$ allows between-cluster communication rounds to ``help out" more when optimizing $\psi_j\rbr{\btheta_j; \{\gamma_i\}_{i \in \cI_j}}$, instead of relying on within-cluster communication. We discuss how they should be adjusted according to the frequency of local and global aggregation rounds, that is, how $\{\tau_j\}_{j = 1}^k$ should be chosen given $\{p_j\}_{j = 0}^k$. 
\begin{proposition}
    \label{prop:opt_tau}
    Suppose $\tau_j = {p_0}(p_0 + 2(1 - p_0)p_j)^{-1}$, $j = 1, \dots, k$. Then
    \begin{multline*}
        \EE_{\{\xi_j\}_{j = 0}^k}[\|G(\btheta) - G(\hat{\btheta})\|^2]\\
        \leq \frac{2}{p_0} \nbr{\nabla_{\btheta}\varphi(\btheta) - \nabla_{\btheta}\varphi(\hat{\btheta})}^2 + \sum_{j = 1}^k \frac{2(1 - \alpha_j)^2}{p_0 + 2(1 - p_0)p_j}\|\nabla_{\btheta_j}\psi_j(\btheta_j) - \nabla_{\btheta_j}\psi_j(\hat{\btheta}_j)\|^2  \\
        + \frac{1}{1 - p_0}\sum_{j = 1}^k \frac{1}{1 - p_j}\|\nabla_{\btheta_j}F_j(\btheta_j) - \nabla_{\btheta_j}F_j(\hat{\btheta}_j)\|^2.
    \end{multline*}
\end{proposition}
\begin{proof}
    See \Cref{subsec:proof_of_prop_opt_tau}.
\end{proof}

Intuitively, $\tau_j$ balances between the probability of a communication round between clusters, $p_0$, and the probability of a communication round within a cluster, $(1 - p_0)p_j$. Therefore, it reduces the equivalent gradient oracle variance. Consider the extreme case where $p_0 \neq 0$ while $p_j = 0$, that is, there are no communication rounds within a cluster. Our choice for $\tau_j$ is then exactly 1, which means that we will optimize $\psi_j\rbr{\btheta_j; \{\gamma_i\}_{i \in \cI_j}}$ only during communication rounds between clusters, which is expected.

\subsection{Convergence Analysis}

We begin our convergence analysis by analyzing the convexity and smoothness of the loss function, $F(\btheta)$. First, we show that the regularizers $\psi_j$ and $\varphi$ are convex and smooth.
\begin{proposition}
    \label{prop:reg_conv}
    The regularizers are convex and smooth. In particular,
    \begin{enumerate}
        \item For all $j$, $\psi_j\left(\btheta_j; \cbr{\gamma_i}_{i \in \cI_j}\right)$ is convex and $\max_{i \in \cI_j} \gamma_i$-smooth in $\btheta_j$.
        \item $\varphi\rbr{\btheta; \cbr{\gamma_i}_{i = 1}^n, \cbr{\alpha}_{j = 1}^k}$ is convex and $\max_{j = 1, \ldots, k}\max_{i \in \cI_j} \alpha_j\gamma_i$-smooth in $\btheta$. 
    \end{enumerate}
\end{proposition}
\begin{proof}
See~\cref{subsec:proof_of_prop_reg_conv}.
\end{proof}

Under \cref{assumption:smooth_and_convex}, from \cref{prop:reg_conv} it follows that the loss function $F(\btheta)$ is $\mu$-strongly convex in $\btheta$ and has a unique minimizer for any set of penalty parameters $(\{\alpha_j\}_{j =1 }^k, \{\gamma_i\}_{i = 1}^n)$. Let $\hat{\Theta}(\{\alpha_j\}_{j = 1}^k, \{\gamma_i\}_{i = 1}^n)$ be the corresponding minimizer. When the penalty parameters $\{\alpha_j\}_{j = 1}^k, \{\gamma_i\}_{i = 1}^n$ are clear from the context, we write
\begin{gather*}
    \hat{\Theta} = \hat{\Theta}(\{\alpha_j\}_{j = 1}^k, \{\gamma_i\}_{i = 1}^n),
    \qquad
    \psi_j(\cdot) = \psi_j\rbr{\cdot;\{\gamma_i\}_{i \in \cI_j}}, 
    \quad j = 1, \ldots, k,\\
    \varphi(\cdot) = \varphi\rbr{\cdot; \{\gamma_i\}_{i = 1}^n, \{\alpha_j\}_{j= 1}^k}, \quad j = 1, \ldots, k.
\end{gather*}
Finally, we use $F_j(\btheta_j) = \sum_{i \in \cI_j}f_i(\theta_i)$ to denote the average unregularized loss function in each cluster. The following theorem provides the convergence rate.
\begin{theorem}[Convergence Rate]
    \label{thm:conv_rate}
    Suppose \cref{assumption:smooth_and_convex} holds and $\{\tau_j\}_{j=1}^k$ are set according to \cref{prop:opt_tau}. Let 
    \begin{equation}
        \label{eqn:es_l}
        \cL = \max\cbr{\frac{2}{p_0}\max_{j = 1, \ldots, k}\max_{i \in \cI_j}\alpha_j\gamma_i, \max_{j = 1,  \dots, k}\frac{2(1 - \alpha_j)\max_{i \in \cI_j}\gamma_i}{p_0 + 2(1 - p_0)p_j}, \frac{L}{1 - p_0}\max_{j = 1,  \dots, k}\frac{1}{1 - p_j}},
    \end{equation} and
    \begin{equation}
    \label{eqn:es_sigma}
    \begin{split}
        \sigma_{\hat{\btheta}}^2 =& \frac{2}{p_0}\|\nabla_{\btheta}\varphi(\hat{\btheta})\|^2 + \sum_{j = 1}^k\frac{2(1 - \alpha_j)^2}{p_0 + 2(1 - p_0)p_j}\|\nabla_{\theta_j}\psi_j(\hat{\btheta}_j)\|^2\\
        &\qquad + \frac{1}{1 - p_0}\sum_{j = 1}^k \frac{1}{1 - p_j}\|\nabla_{\btheta_j}F_j(\hat{\btheta}_j)\|^2.
    \end{split}
    \end{equation}
    If the step size satisfies $\eta \leq \frac{1}{2\cL}$, then 
    \[
        \EE[\|\btheta^t - \hat{\btheta}\|^2] \leq \rbr{1 - \eta \mu}^t\|\btheta^0 - \hat{\btheta}\|^2 + \frac{2\eta \sigma_{\hat{\btheta}}^2}{\mu}.
    \]
\end{theorem}
\begin{proof}
    See \cref{subsec:proof_of_thm_conv_rate}.
\end{proof}

Compared with the single cluster result in~\cite{hanzely2020federated}, our convergence rate involves a few additional terms. By allowing each cluster to determine when to communicate individually, the expected smoothness coefficient and the variance of the gradient oracle at the optimum are more complex, as they incorporate both local and global communication frequencies, $\{p_j\}_{j = 1}^k$ and $p_0$. When we set $p_j= 0, \alpha_j = 1$ for all $j$, we recover the single cluster convergence rates given in Theorem 4.5 of \cite{hanzely2020federated} up to constant factors.

The parameters $p_j$, $j = 0, \dots, k$, that control the 
frequency of communication can be tuned by minimizing $\cL$ in principle. Unfortunately, as $\cL$ is effectively a maximum taken over $2k + 1$ different terms, directly minimizing the expression is infeasible. We instead consider minimizing the following upper bound on $\cL$: 
\begin{equation}
\label{eqn:es_l_tilde}
\Tilde{\cL} = \max\cbr{\frac{2}{p_0}C_1, \max_{j = 1, \ldots, k}\frac{2C_2}{p_0 + 2(1 - p_0)p_j}, \frac{L}{1 - p_0}\max_{j = 1, \ldots,k}\frac{1}{1 - p_j}},
\end{equation}
where $C_1 = \max_{j = 1, \ldots, k}\max_{i \in \cI_j}\alpha_j\gamma_i$ and $C_2 = \max_{j = 1, \ldots, k}\max_{i \in \cI_j}(1 - \alpha_j)\gamma_i$.
The choice of parameters depends on the relationship between $C_1$ and $C_2$.

\begin{corollary}
    \label{corollary:opt_conv_rate}
Suppose that $\{\tau_j\}_{j = 1}^k$ are set according to \cref{prop:opt_tau}.

When $C_2 > C_1$, setting $\eta = \frac{1}{2\cL}$, $p_0 = \frac{2C_1}{C_1 + C_2 + L}$, and $p_j = \frac{C_2 - C_1}{C_2 - C_1 + L}$ ensures that the optimal number of iterations is in $\cO\rbr{\frac{(C_1 + C_2 + L)}{\mu}\log \frac{1}{\epsilon}}$, the number of communication rounds between clusters is in $\cO\rbr{\frac{C_1(C_2 - C_1 + L)}{(C_1 + C_2 + L)\mu}\log \frac{1}{\epsilon}}$, and the number of communication rounds within a cluster is in $\cO\rbr{\frac{L(C_2 - C_1)}{(C_2 - C_1 + L)\mu}\log\frac{1}{\epsilon}}$ for all clusters.    
    
When $C_2 \leq C_1$, setting $\eta = \frac{1}{2\cL}$, $p_0 = \frac{2C_1}{2C_1 + L}$, and $p_j = 0$ ensures that the optimal number of iterations is in $\cO\rbr{\frac{(C_1 + L)}{\mu}\log \frac{1}{\epsilon}}$, the number of communication rounds between clusters is in $\cO\rbr{\frac{C_1 L}{(C_1 + L)\mu}\log \frac{1}{\epsilon}}$, and the number of communication rounds within a cluster is 0.
\end{corollary}
\begin{proof}
    See \cref{subsec:proof_of_corollary_opt_conv_rate}.
\end{proof}

We conclude the section by emphasizing that~\cref{alg:async_l2gd} does not require a central server to aggregate information across all clusters. The central server in~\cref{alg:async_l2gd} only serves two purposes: flipping a coin ($\xi_0$) to determine whether a communication round between clusters is necessary and calculating the global average should there be a communication round between clusters. The former can be easily decentralized across clusters by asking all cluster servers to flip a coin and take a majority vote, while the latter can be implemented by asking all cluster servers to communicate with one another.

\section{Asynchronous Accelerated Loopless Local SGD with Variance Reduction} \label{app:alg_variants}

The convergence rate of~\cref{alg:async_l2gd} when minimizing the loss in~\cref{eqn:loss} is suboptimal due to the lack of acceleration and variance reduction. In this section, we propose an accelerated, variance reduced, stochastic variant of \cref{alg:async_l2gd} tailored to the finite-sum setting. The variant enjoys optimal communication complexity in a single cluster setting \cite{hanzely2020lower}, regardless of the relationship between regularization strength and smoothness of loss functions $f_i$, outperforming popular alternatives discussed in~\cite{dinh2020personalized, hanzely2021personalized, li2021ditto, mansour2020three}. We further hypothesize that the optimality holds when extended to multi-cluster setup studied here.
\floatname{algorithm}{\color{black}Algorithm}
\begin{algorithm}[htbp]
\caption{\color{black}Async-AL2SGD+}
\label{alg:async_al2sgd_plus}
\begin{algorithmic}
    \STATE \textbf{Input:} Step size $\eta$, probabilities $p_1, \dots, p_k, p_0, \rho \in [0, 1]$, fractions $\tau_1, \dots, \tau_k \in [0, 1]$. 
    \STATE \textbf{Initialize:} $1 < a_0, a_1 < 1$, $b_1, b_2 > 0$, $x_i^0 = y_i^0 = z_i^0 = \theta_i^0 = \bm{0}_d \in \RR^d$
    \FOR{$t = 1, 2, \dots$}
    \STATE All \textcolor{red}{Clients} $i = 1, \dots, n$ perform local update\\
    \quad $\theta_i^t = a_1z_i^t + a_2x_i^t + (1 - a_1 - a_2)y_i^t$.
    \STATE $\xi_0 = 1$ with probability $p_0$ and 0 with probability $1 - p_0$
    \IF {$\xi_0 = 1$}
        \STATE All \textcolor{blue}{Clusters} $j = 1, \dots, k$ compute cluster average $\Bar{\theta}_j^t = \frac{\sum_{i \in \cI_j}\gamma_i \theta_i^t}{\sum_{i \in \cI_j}\gamma_i}$
        \STATE All \textcolor{blue}{Clusters} aggregate network average $\Bar{\theta}^t = \frac{\sum_{j = 1}^k \sum_{i \in \cI_j} \alpha_j \gamma_i \theta_i^t}{\sum_{j = 1}^k \sum_{i \in \cI_j} \alpha_j \gamma_i}$
        \STATE All \textcolor{red}{Clients} calculate gradient estimate according to \cref{eqn:async_al2sgd_plus_gradient_oracle}
        \STATE Set $y_i^{t + 1} = \theta_i^t - \eta g_i^t$
    \ELSE
        \FOR{All \textcolor{blue}{Clusters} $j = 1, \dots, k$ in parallel}
            \STATE $\xi_j = 1$ with probability $p_j$ and 0 with probability $1 - p_j$
            \IF{$\xi_j = 1$}
                \STATE Compute cluster average $\Bar{\theta}_j^t = \frac{\sum_{i \in \cI_j}\gamma_i \theta_i^t}{\sum_{i \in \cI_j}\gamma_i}$ and send it back to each client
                \STATE All \textcolor{red}{Clients} calculate gradient estimate according to \cref{eqn:async_al2sgd_plus_gradient_oracle}
                \STATE Set $y_i^{t + 1} = \theta_i^t - \eta g_i^t$
            \ELSE
                \STATE All \textcolor{red}{Clients} calculate gradient estimate according to \cref{eqn:async_al2sgd_plus_gradient_oracle}
                \STATE Set $y_i^{t + 1} = \theta_i^t - \eta g_i^t$
            \ENDIF
        \ENDFOR
    \ENDIF
    \STATE All \textcolor{red}{Clients} $i = 1, \dots, n$ update: $z_i^{t + 1} = b_1z_i^t + (1 - b_1)\theta_i^t + \frac{b_2}{\eta}(y_i^{t + 1} - \theta_i^t)$
    \STATE $\xi' = 1$ with probability $\rho$ and 0 with probability $1 - \rho$
    \IF{$\xi' = 0$}
        \STATE For all \textcolor{red}{Clients} $i = 1, \dots, n$: $x_i^{t + 1} = x_i^t$
    \ELSE
        \STATE For all \textcolor{red}{Clients} $i = 1, \dots, n$ update $x_i^{t + 1} = y_i^{t + 1}$, and evaluate and store $\nabla f_i(x_i^{t + 1})$
        \STATE All \textcolor{blue}{Clusters} communicate, compute averages $\Bar{x}^t = \frac{\sum_{j = 1}^k \sum_{i \in \cI_j} \alpha_j \gamma_i x_i^t}{\sum_{j = 1}^k \sum_{i \in \cI_j} \alpha_j \gamma_i}$, $\Bar{x}_j^t = \frac{\sum_{i \in \cI_j}\gamma_i x_i^t}{\sum_{i \in \cI_j}\gamma_i}$ for all $j$, and send them back to the clients.
    \ENDIF
    \ENDFOR
\end{algorithmic}
\end{algorithm}

We assume that the local loss has a finite sum structure over smooth and strongly convex functions, a common assumption in the literature on accelerated variance reduced algorithms \citep{hanzely2020lower, hanzely2020federated, kovalev2020don}. We formally characterize our assumption below.
\begin{assumption}
\label{assumption:strong_conv_smooth_ae}
    The loss function $f_i$, $i=1,\ldots,n$, has the following finite structure:
    \begin{equation*}
        f_i(\theta) = \frac{1}{n_i}\sum_{l = 1}^{n_i}\Tilde{f}_{i, l}(\theta),
    \end{equation*} 
    where $\tilde f_{i,l}$ is $\Tilde{L}$-smooth and $\mu$-strongly convex, $l=1,\ldots,n_i$.
\end{assumption}

We define a stochastic gradient estimate for all clients, similar to the construction for~\cref{alg:async_l2gd}. For a client $i$ that belongs to the cluster $j$, the variance reduced stochastic gradient is
\begin{equation}
    \label{eqn:async_al2sgd_plus_gradient_oracle}
    \begin{split}
        g_i^t &= \nabla f_i(x_i^t) + \alpha_j\gamma_i(x_i^t - \Bar{x}^t) + (1 - \alpha_j)\gamma_i(x_i^t - \Bar{x}_j^t)\\
        &\quad \quad + \ind\{\xi_0 = 1\} \frac{\gamma_i\alpha_j}{p_0}(\theta_i^t - \Bar{\theta}^t - (x_i^t - \Bar{x}^t))\\
        &\quad \quad + \ind\{\xi_0 = 0\}\ind\{\xi_j = 1\} \frac{\gamma_i(1 - \tau_j)(1 - \alpha_j)}{(1 - p_0)p_j}\rbr{(\theta_i^t - \Bar{\theta}_j^t) - (x_i^t - \Bar{x}_j^t)}\\
        &\quad \quad + \ind\{\xi_0 = 0\}\ind\{\xi_j = 0\} \frac{1}{(1 - p_0)(1 - p_j)}\rbr{\nabla \Tilde{f}_{i, l}(\theta_i^t) - \nabla\Tilde{f}_{i, l}(x_i^t)},
    \end{split}
\end{equation} 
where $l$ is selected uniformly at random at each iteration for every client.
At a high level, \cref{eqn:async_al2sgd_plus_gradient_oracle} defines a stochastic gradient oracle for a finite-sum composite optimization problem. The different realizations of $\{\xi_0, \ldots, \xi_j\}$ determine the type of communication round to execute, if any, at any given step. For example, when $\xi_0 = 1$, a communication round between the client clusters is executed, while when $x_0 = 0$ and $\xi_1 = 1$, the first client cluster executes a communication round within the cluster. The following lemma provides a bound on the variance of the stochastic gradient oracle in~\cref{eqn:async_al2sgd_plus_gradient_oracle}.

\begin{lemma}
\label{lemma:al2sgd_gradient_oracle_expected_smoothness}
Suppose that~\cref{assumption:strong_conv_smooth_ae} holds and that $\{\tau_j\}_{j = 1}^k$ are selected as in
\cref{prop:opt_tau}.
Let
\[
    \cL = \max\cbr{\frac{2}{p_0}\max_{j = 1, \ldots, k}\max_{i \in \cI_j}\alpha_j\gamma_i, \max_{j = 1, \dots, k}\frac{2(1 - \alpha_j)\max_{i \in \cI_j}\gamma_i}{p_0 + 2(1 - p_0)p_j}, \frac{\Tilde{L}}{1 - p_0}\max_{j = 1,  \dots, k}\frac{1}{1 - p_j}}.
\]
Then 
\begin{equation}
    \label{eqn:al2sgd_gradient_oracle_bound}
    \EE\sbr{\|\bg^t - \nabla F(\bx^t)\|^2} \leq 2 \cL D_F(\btheta^t, \bx^t),
\end{equation}
where $\bg^t = (g_1^t, \ldots, g_n^t)^T \in \RR^{nd}$ is the variance reduced stochastic gradient oracle, where $g_i^t$ is defined in \cref{eqn:async_al2sgd_plus_gradient_oracle}, and $D_F(x_1, x_2) \coloneq F(x_1) - F(x_2) - \langle \nabla F(x_2), x_1 - x_2 \rangle$ is the Bregman divergence induced by the loss function $F$ in~\cref{eqn:loss}.
\end{lemma}
\begin{proof}
    See \cref{subsec:proof_of_lemma_al2sgd_gradient_oracle_expected_smoothness}.
\end{proof}

We use the oracle to construct an instance of L-Katyusha, a loopless, variance-reduced, accelerated algorithm~\cite{qian2021svrg}. \Cref{alg:async_al2sgd_plus} provides the pseudocode, while \cref{thm:al2sgd_conv_rate} provides the convergence rate for \cref{alg:async_al2sgd_plus}.

\begin{theorem}\label{thm:al2sgd_conv_rate}
Suppose that the conditions of~\cref{lemma:al2sgd_gradient_oracle_expected_smoothness} are satisfied.
Let
\begin{gather*}
    L_F = \Tilde{L} + \max_{i = 1, \ldots, n}\gamma_i,
    \quad
    \eta = \frac{1}{4}\max\{L_F, \cL\}^{-1},\\
    a_1 = \min\cbr{\frac{1}{2}, \sqrt{\eta\mu\max\cbr{\frac{1}{2}, \frac{a_2}{\rho}}}}, \quad
    a_2 = \frac{\cL}{2\max\{L_F, \cL\}},\\
    b_1 = 1 - b_2\mu, \quad 
    b_2 = \frac{1}{\max\{2\mu, 4a_1/\eta\}}.
\end{gather*} 
Then the iteration complexity of~\cref{alg:async_al2sgd_plus} is
\[
    \cO\rbr{\rbr{\frac{1}{\rho} + \sqrt{\frac{\cL}{\rho\mu}}}\log\frac{1}{\epsilon}}.
\]
\end{theorem}
\begin{proof}
    See \cref{subsec:proof_of_thm_al2sgd_conv_rate}.
\end{proof}
In addition to the iteration complexity, it is easy to obtain communication bounds, gradient complexity, as well as optimal parameters. Similar to \cref{corollary:opt_conv_rate}, directly minimizing $\cL$ over all the parameters is infeasible and we consider the following upper bound on $\cL$ instead:
\[
    \Tilde{\cL} = \max\cbr{\frac{2C_1}{p_0}, \frac{2C_2}{p_0 + 2(1 - p_0)p_j}, \frac{\Tilde{L}}{1 -p_0}\max_{j = 1, \ldots, k}\frac{1}{1 - p_j}},
\] 
where $C_1 = \max_{j = 1, \ldots, k}\max_{i \in \cI_j}\alpha_j\gamma_i$ and $C_2 = \max_{j = 1, \ldots, k}\max_{i \in \cI_j}(1 - \alpha_j)\gamma_i$. The upper bound is virtually the same as the one in \cref{eqn:es_l_tilde}, and we have the following.

\begin{corollary}\label{corollary:async_al2sgd_plus_opt_conv_rate}
Consider~\cref{alg:async_al2sgd_plus} with a fixed $\rho$ and the tuning parameters set as:
\begin{gather*}
    L_F = \Tilde{L} + \max_{i = 1, \ldots, n}\gamma_i,
    \quad
    \eta = \frac{1}{4}\max\{L_F, \cL\}^{-1},\\
    a_1 = \min\cbr{\frac{1}{2}, \sqrt{\eta\mu\max\cbr{\frac{1}{2}, \frac{a_2}{\rho}}}}, \quad
    a_2 = \frac{\cL}{2\max\{L_F, \cL\}},\\
    b_1 = 1 - b_2\mu, \quad 
    b_2 = \frac{1}{\max\{2\mu, 4a_1/\eta\}}.
\end{gather*} 
Furthermore, $\{\tau_j\}_{j = 1}^k$ is set according to \cref{prop:opt_tau}.

When $C_2 > C_1$, setting $p_0 = \frac{2C_1}{C_1 + C_2 + \Tilde{L}}$ and $p_j = \frac{C_2 - C_1}{C_2 - C_1 + \Tilde{L}}$ ensures that the optimal number of iterations is in $\cO\rbr{\sqrt{\frac{C_1 + C_2 + \Tilde{L}}{\mu}}\log\frac{1}{\epsilon}}$,
the number of communication rounds between clusters and within a cluster is in $\cO\rbr{\frac{C_1(C_2 - C_1 + \Tilde{L})}{(C_1 + C_2 + \Tilde{L})\sqrt{(C_1 + C_2 + \Tilde{L})\mu}}\log\frac{1}{\epsilon}}$ and $\cO\rbr{\frac{(C_2 - C_1)\Tilde{L}}{(C_2 - C_1 + \Tilde{L})\sqrt{(C_1 + C_2 + \Tilde{L})\mu}}\log\frac{1}{\epsilon}}$, respectively.
     
When $C_2 \leq C_1$, setting $p_0 = \frac{2C_1}{2C_1 + \Tilde{L}}$ and $p_j = 0$ ensures the optimal number of iterations is in $\cO\rbr{\sqrt{\frac{C_1 + \Tilde{L}}{\mu}}\log\frac{1}{\epsilon}}$, the number of communication rounds between clusters is in $\cO\rbr{\frac{C_1\Tilde{L}}{(C_1 + \Tilde{L})\sqrt{(C_1 + \Tilde{L})\mu}}\log\frac{1}{\epsilon}}$, and the number of communication rounds within a cluster is 0.
\end{corollary}
\begin{proof}
    See \cref{subsec:proof_of_corollary_asyn_al2sgd_plus_opt_conv_rate}.
\end{proof}
\section{Case Study: Hierarchical Linear Model}\label{sec:hlm_setup}

In the previous section, we have answered how to minimize the objective in \cref{eqn:loss}. Next, we provide a statistical model of personalization under which the minimizer of \cref{eqn:loss} corresponds to an estimator that outperforms the common alternatives. More precisely, we show that the minimizer of \cref{eqn:loss} strictly outperforms both (a) training a single global model for all clients and (b) training a separate model for each client independent of the data of other clients. Unlike the analysis in \cite{li2021ditto}, we consider the hierarchical, multi-cluster regime. Existing approaches to personalized federated learning often use loss functions similar to those discussed in \cref{eqn:loss,eqn:mtl_loss}. Therefore, it is important to understand the statistical properties of the corresponding minimizers.

Although the minimizer~\cref{eqn:loss} outperforms commonly used alternatives, we also provide two alternative estimators that are hard to efficiently compute in a federated learning setting, yet dominate our proposed estimator. Specifically, they achieve a lower mean squared error. The efficient implementation of the two alternatives remains an open question for future research.

The statistical model we consider in this section is based on a hierarchical linear model with Gaussian priors \citep{stephen2002hierarchical}. Nature first draws the cluster centers from a Gaussian distribution with unknown mean and then draws each client's parameter from a Gaussian distribution centered at the cluster center the client belongs to. More precisely, for an unknown parameter $\Bar{\theta}^* \in \RR^d$, our model is:
\begin{equation}\label{eqn:hlm}
\begin{array}{lll}
    \Bar{\theta}_j^* = \Bar{\theta}^* + \Bar{\xi}_j, 
    & \Bar{\xi}_j \sim \cN(0, \Bar{\sigma}^2I_d), 
    & j = 1, \ldots, k,\\
    \theta_i^* = \Bar{\theta}_j^* + \xi_i, 
    & \xi_i \sim \cN(0, \Bar{\sigma}_j^2 I_d), 
    & i \in \cI_j,\\
    y_i = X_i\theta_i^* + \epsilon_i, 
    & \epsilon_i \sim \cN(0, \sigma_i^2 I_{n_i}), 
    & i \in \cI_j,
\end{array}
\end{equation} 
where $\Bar{\theta}_j^* \in \RR^d$ represents the center of the cluster $j$, 
$\theta_i^* \in \RR^d$ represents the client-specific parameter, and $(X_i, y_i) \in \RR^{n_i \times d} \times \RR^{n_i}$ is the data set on the $i$-th client.

When estimating all client parameters simultaneously in \cref{eqn:hlm}, we obtain the following maximum likelihood estimation problem:
\begin{align*}
    \min_{\{\theta_i\}_{i = 1}^m} \sum_{j = 1}^k\sum_{i \in \cI_j} \left(\frac{1}{\sigma_i^2}\|y_i - X_i \theta_i\|^2 + \frac{\gamma_i\alpha_j}{2}\|\theta_i - \Bar{\theta}\|^2 + \frac{\gamma_i(1 - \alpha_j)}{2}\|\theta_i - \Bar{\theta}_j\|\right),
\end{align*}
where 
\begin{align*}
    \Bar{\theta} &= \rbr{\sum_{j' = 1}^k \sum_{i' \in \cI_{j'}} \gamma_{i'} \alpha_{j'}}^{-1}\sum_{j= 1}^k \sum_{i = \in \cI_j}\gamma_i\alpha_j\theta_i; \\
    \Bar{\theta}_j &= \rbr{\sum_{i' \in \cI_j}\gamma_{i'}}^{-1}\sum_{i\in\cI_j}\gamma_i \theta_i, \ j = 1, \ldots, k.
\end{align*}
The objective is an instance of~\cref{eqn:loss}, and by \cref{prop:equiv_mtl}, is equivalent to
\begin{equation}
    \label{eqn:hlm_loss}
    \min_{\{\theta_i\}_{i = 1}^m, \{w_j\}_{j = 1}^k, \Bar{w}}\quad \sum_{j = 1}^k\left(\frac{\lambda_j}{2}\|w_j - \Bar{w}\|^2 + \sum_{i \in \cI_j}\left(\frac{1}{2\sigma_i^2}\|y_i - X_i\theta_i\|^2 + \frac{\gamma_i}{2}\|\theta_i - w_j\|^2\right)\right).
\end{equation}
We focus on \cref{eqn:hlm_loss} for convenience and show that, when $\{\lambda_j\}_{j=1}^k$, $\{\gamma_i\}_{i = 1}^n$ are properly tuned and $\Bar{\sigma}^2$, $\Bar{\sigma}_j^2$ are known, the resulting minimizers $\{\hat{\theta}_i\}_{i = 1}^n$ attain the smallest mean squared error among a class of linear unbiased estimators \citep{kariya2004generalized}.
\begin{theorem}\label{thm:blue}
    Suppose that $\lambda_j = \rbr{\Bar{\sigma}^2}^{-1}$,
    $\gamma_i = \rbr{\Bar{\sigma}_j^2}^{-1}$, and $X_i^TX_i = \beta_i I_d$ for some $\beta_i \in \RR$,
    $i \in \cI_j$, $j = 1, \ldots, k$. Then $\hat{\theta}_i$, obtained as the minimizer of \cref{eqn:hlm_loss}, is the best linear unbiased estimator of $\theta_i^*$ given $\{(X_i, y_i)\}\cup\{\hat{\theta}_{i'}^d\}_{i' \neq i}$, where
    $\hat{\theta}_i^d = (X_i^TX_i)^{-1}(X_iy)$.
\end{theorem}
\begin{proof}
    See \cref{subsec:proof_of_thm_blue}.
\end{proof}

To provide an intuition behind \cref{thm:blue}, we first consider the single cluster setting in \cite{li2021ditto}. Focusing on an arbitrary cluster $j$, we know that for all $i \in \cI_j$, $\hat{\theta}_i$ consists of two parts: the first part estimates $\theta_i^*$ using only data on the $i$-th client and the second part estimates $\Bar{\theta}_j^*$ using other clients' parameters. Since \cref{eqn:hlm_loss} aggregates information across clients only by regularizing the distance between weight estimates, the estimate for $\Bar{\theta}_j^*$ cannot depend directly on $\{(X_{i'}, y_{i'})\}_{i' \neq i}$ and is constructed instead through $\{\hat{\theta}_{i'}^d\}_{i' \neq i}$. In our multi-cluster setting, the estimator $\hat{\theta}_i$ operates in a similar fashion, as can be seen from the proof of \cref{thm:blue}. Since we do not have direct access to the data of other clients, the estimates for $\Bar{\theta}_j^*, \Bar{\theta}^*$ are constructed indirectly using $\{\hat{\theta}_{i'}^d\}_{i' \neq i}$. A direct consequence of \cref{thm:blue} is the following.

\begin{corollary}
\label{corollary:outperform_all_decentralized}
Suppose that the conditions of \Cref{thm:blue} hold. 
Let $\hat{\theta}^{\text{all}} = \rbr{\sum_{i= 1}^n X_i^TX_i}^{-1}\rbr{\sum_{i = 1}^n X_iy_i}$.
Then 
\[
    \EE[\|\hat{\theta}_i - \theta_i^*\|^2] \leq \min\cbr{\EE[\|\hat{\theta}_i^d - \theta_i^*\|^2], \EE[\|\hat{\theta}^{all} - \theta_i^*\|^2]},
\] 
where $\hat{\theta}_i$ and $\hat{\theta}_i^d$ are 
defined in \Cref{thm:blue}, and the expectation is taken over the randomness in \cref{eqn:hlm}.
\end{corollary}

\cref{corollary:outperform_all_decentralized} illustrates  existence of a regime under which a personalized estimator consistently outperforms the alternatives, learning a single model for all clients without any personalization $\hat{\theta}^{all}$ and learning a model independently for each client $\hat{\theta}_i^d$, proposed in~\cite{chen2021theorem}, highlighting the effectiveness and necessity of personalization within highly structured problems. Note that $\hat{\theta}_i^d$ and $\hat{\theta}^{all}$ can both be written as unbiased linear estimators of $\theta_i^*$ given $\{(X_i, y_i)\} \cup \{\hat{\theta}_{i'}^d\}_{i' \neq i}$. The optimality of $\hat{\theta}_i$ among this class of estimators ensures that its mean squared error is no greater than these alternatives.

\subsection{Limitations of Unbiased Estimators}
\label{subsec:limitations_of_unbiased_estimators}

We have shown that solving~\cref{eqn:hlm_loss} consistently outperforms common alternatives $\hat{\theta}^{all}$ and $\hat{\theta}_i^d$. In particular,   regularizing the distance between client model parameters and average model parameters provides a viable approach for personalization. See also \cite{hanzely2020federated, li2021ditto, hanzely2021personalized, dinh2020personalized}. Furthermore, our result complements~\cite{chen2021theorem}, identifying a regime in which personalization consistently outperforms learning a single model and learning models independently. Unfortunately, we cannot guarantee that our approach is optimal among all possible estimators. We discuss two alternative estimators that result in a lower mean squared error, but are hard to implement in federated learning setting.

First, we note that~\cref{eqn:hlm_loss} is the maximum likelihood estimator when we simultaneously estimate all clients' parameters. However, for any particular client, we can derive an unbiased linear estimator with smaller mean squared error,
by marginalizing other clients' and clusters' parameters. 
\begin{proposition}
    \label{prop:not_gls}
    Suppose $\{(X_i, y_i)\}_{i = 1}^n$ are generated according to the model in 
    \cref{eqn:hlm}. For any $i$, there exists a linear unbiased estimator $\Tilde{\theta}_i$ of $\theta_i^*$ that satisfies $\EE[\|\Tilde{\theta}_i - \theta_i^*\|^2] \leq \EE[\|\hat{\theta}_i - \theta_i^*\|^2]$. 
\end{proposition}
\begin{proof}
    See \cref{subsec:proof_of_prop_not_gls}. 
\end{proof} 

By marginalizing other parameters, we derive $\Tilde{\theta}_i$, the best linear unbiased estimator of $\theta_i^*$ given $\{(X_{i'}, y_{i'})\}_{i' = 1}^n$, which includes data from other clients. An explicit form for the equation that $\Tilde{\theta}_i$ solves can be found in \cref{subsec:proof_of_prop_not_gls}. By contrast, $\hat{\theta}_i$ only has access to other clients' weight estimates, $\{\hat{\theta}_{i'}^d\}_{i' \neq i}$, but does not have direct access to their data. We emphasize, however, that doing so for all the clients is costly, as we need to solve a separate generalized least squares problem for all clients. On the other hand, learning $\hat{\theta}_i$ for all clients can be done simultaneously by optimizing~\cref{eqn:hlm_loss}. 

An alternative estimator of $\theta_i^*$ can be constructed based on James-Stein estimator~\cite{james1992estimation}. Suppose that $X \sim N(\xi, I_p)$ is a $p$-dimensional Gaussian random vector with mean $\xi$ and covariance $I_p$.
The James-Stein Estimator of $\xi$ is defined as
\[
    \hat{\xi}^{\rm JS} = \rbr{1 - \frac{p - 2}{\|X\|_2^2}}X.
\]
Let $\hat{\xi}^{\rm MLE}$ be the maximum likelihood estimator of $\xi$, that is, $\hat{\xi}^{\rm MLE} = X$. Then, for all $p \geq 3$,
\[
    \EE\sbr{\nbr{\hat{\xi}^{\rm JS} - \xi}^2} \leq \EE\sbr{\nbr{\hat{\xi}^{\rm MLE} - \xi}^2}.
\]
See \cite{james1992estimation} for a proof.

We construct a biased estimator that dominates $\hat{\theta}_i$ even in a simplified, single cluster regime, under which we are effectively solving a $d$-dimensional point estimation problem. 
\begin{proposition}\label{prop:JS}
Consider a single-cluster model under which $y_i \sim N(\theta_i^*, I_d)$, where $\theta_i^* \sim N(\Bar{\theta}^*, I_d)$, $i = 1, \dots, n$, and $\Bar{\theta}^* \in \RR$ is an unknown parameter. If $d > 3$, then there exists a biased estimator $\Tilde{\theta}^{JS}_i$ such that $\EE[\|\Tilde{\theta}_i^{JS} - \theta_i^*\|^2] \leq \EE[\|\hat{\theta}_i - \theta_i^*\|^2]$, where $\hat{\theta}_i$ is the best linear unbiased estimator for $\theta_i^*$ given $\{y_i\}_{i = 1}^n$. 
\end{proposition}
\begin{proof}
    See \cref{subsec:proof_of_prop_JS}.
\end{proof}

Efficiently implementing the biased estimator is non-trivial. The estimator requires careful adjustment of a shrinkage coefficient to achieve a smaller mean squared error. How to tune this shrinkage coefficient efficiently in a federated learning setting is unclear. As a result, the minimizer of~\cref{eqn:hlm_loss} is a great practical alternative.

The two estimators provided in this section also point out that alternative approaches to personalized federated learning~\citep{deng2020adaptive, li2021ditto, hanzely2020federated} are not optimal from a statistical point of view under a hierarchical linear model. While these two estimators are impractical in federated learning setting, it remains an open question how to approximate them with efficient computation and communication, while also respecting the privacy concerns in federated learning.

\section{Numerical Results}
\label{sec:experiments}
 
We illustrate the performance of our algorithm on both simulated data and a real-world marketing data set. We focus on the generalization error of the minimizer of \cref{eqn:loss}, rather than on the optimization performance of \cref{alg:async_l2gd,alg:async_al2sgd_plus}, since related approaches have been studied in the single-cluster regime \cite{li2021ditto, hanzely2020federated, hanzely2020lower, hanzely2021personalized, dinh2020personalized}. Our aim is to complement those studies.

We consider a modified version of~\cref{alg:async_l2gd}, where one coin toss is used to determine whether the clients should perform a local step or a communication round. During a communication round, each machine minimizes its distance to the cluster and global average simultaneously. Our experiments focus on the generalization behavior of the estimator rather than on the optimization error. Therefore, such a simplification does not affect the validity of our results and simplifies the implementation. Scripts for replicating the experiments can be found \href{https://github.com/ShawnBLYU/Personalized-Federated-Learning-with-Multiple-Known-Clusters/tree/main}{this GitHub repository}.

\subsection{Simulation Studies}

We compare our algorithm with the tuning parameters set as in \cref{thm:blue} against three baselines: i) learning a single model for all clients, ii) learning each client's model independently, and iii) learning personalized models centered around a single point~\cite{hanzely2020federated, li2021ditto, dinh2020personalized}.

Simulation data are generated from the following hierarchical linear model. The center of all clusters is $\Bar{\theta}^* = \mathbf{0}_{20} \in \RR^{20}$, where $\mathbf{0}_{20}$ is an all-zero vector. There are 20 clusters, each with 20 clients. The cluster centers and client parameters are generated as:
\begin{align*}
    &\Bar{\theta}_j^* \sim \cN(\Bar{\theta}^*, I_{20}), \quad j = 1, \dots, k,\\
    &\theta_i^* \sim \cN(\Bar{\theta}_j^*, I_{20}), \quad i \in \cI_j.
\end{align*} 
On each client, we generate a data matrix $X_i \in \RR^{m \times 20}$, where $m$ is the number of samples on the $i$-th client and is selected from $\{1, 5, 25, 50, 100, 200\}$. Each entry in $X_i$ is drawn i.i.d.~from the standard Gaussian distribution, $\cN(0, 1)$. Subsequently, the response $y_i$ is drawn from the linear model:
\[
    y_i \mid X_i \sim \cN(X_i\theta_i^*, I_{n_i}).
\]  

Under the data generating procedure described above, the three baselines take the following form:
\begin{enumerate}
    \item Training a \textbf{single-model} for all clients:  $\hat{\theta}^{\mathbf{sm}} = \rbr{\sum_{i=  1}^n X_i^TX_i}^\dagger\rbr{\sum_{i = 1}^nX_i^Ty_i}$.

    \item Entirely \textbf{locally-trained} estimator: for each client $i$,  $\hat{\theta}_i^{\mathbf{lt}} = (X_i^TX_i)^\dagger X_i^Ty_i$, where $(X_i^TX_i)^\dagger$ is the Moore-Penrose pseudo-inverse of the empirical covariance matrix.

    \item Training a \textbf{single-cluster} personalized model~\citep{hanzely2020federated, li2021ditto, dinh2020personalized}: the objective function is given by
    \[
        \min_{\{\theta_i\}} \frac{1}{n}\sum_{i = 1}^n \rbr{\frac{1}{2}\|X_i\theta_i - y_i\|^2 + \frac{\lambda^{\mathbf{sc}}}{2}\|\theta_i - \Bar{\theta}\|^2},
    \] where $\Bar{\theta} = n^{-1}\sum_{i = 1}^n \theta_i$ and $\lambda^{\mathbf{sc}} > 0$ is user-chosen parameter that controls the strength of personalization. The minimizer of the objective can be obtained as 
    $\hat{\theta}_i^{\mathbf{sc}} = (X_i^TX_i + \lambda^{\mathbf{sc}}I_{20})^{-1}\rbr{(X_i^TX_i)\hat{\theta}_i^{\mathbf{lt}} + \lambda^{\mathbf{sc}}\hat{\Bar{\theta}}^{\mathbf{sc}}}$,
    where
    \begin{multline*}
        \hat{\Bar{\theta}}^{\mathbf{sc}} = \rbr{I_{20} - \frac{\lambda^{\textbf{sc}}}{n}\sum_{i = 1}^n\rbr{X_i^TX_i + \lambda^{\mathbf{sc}}I_{20}}^{-1}}^{-1}\rbr{\frac{1}{n}\sum_{i = 1}^n \rbr{X_i^TX_i + \lambda^{\mathbf{sc}} I _{20}}^{-1}(X_i^TX_i)\hat{\theta}_i^{\mathbf{lt}}}.
    \end{multline*}
    We tune $\lambda^{\mathbf{sc}}$ over 20 evenly spaced points in $[0.01, 2]$ using cross validation.
\end{enumerate}

Following \cref{thm:blue}, we set $\lambda_j = 1$ for all $j$ and $\gamma_i = 1$ for all $i$. We use \cref{prop:equiv_mtl} to convert the maximum likelihood estimation problem into the form in \cref{eqn:loss}, which we minimize using the simplified algorithm with communication probability $p = 0.1$, stepsize $\eta = 10^{-4}$, and maximum number of iterations $50000$. We use $\{\hat{\theta}_i^{\mathbf{our}}\}$ to denote the estimators produced by our model. 

We measure the performance of different estimators using the $\ell_2$ distance between the estimates of the parameters of the clients and their actual parameters, that is, $\|\hat{\theta}_i - \theta^*_i\|_2^2$ for $\hat{\theta}_i \in \{\hat{\theta}^{\mathbf{sm}}, \hat{\theta}^{\mathbf{lt}}_i, \hat{\theta}_i^{\mathbf{sc}}, \hat{\theta}_i^{\mathbf{our}}\}$.  
The results are averaged over five independent runs. 

To demonstrate that our model consistently outperforms baselines, we pick two specific choices of $m$, $m = 10$, and $m = 100$. Intuitively, as $m$ increases, local training becomes more viable, whereas a smaller $m$ means that training a single model could be more beneficial. Here, we show that our suggested approach outperforms both these alternatives regardless of $m$. \Cref{fig:sim_mse_n} and \Cref{tab:sum_n} confirm that our proposed method consistently outperforms alternatives for $m \in \{1, 5, 10, 25, 100, 200\}$.

\begin{table}[hb]
    \centering
    \begin{tabular}{|l|c|c|c|c|}
    \hline
         & \multicolumn{2}{c|}{$m = 10$} & \multicolumn{2}{c|}{$m = 100$}\\
        & {\small Avg. ($\pm$ SD.)} & Max & Avg. ($\pm$ SD.) & Max\\
    \hline
        $\hat{\theta}^{\mathbf{lt}}$ & 4.50 $(\pm 0.981)$ & 8.224 & 0.494 $(\pm 0.093)$ & 0.768\\
        $\hat{\theta}^{\mathbf{sm}}$ & 6.11 $(\pm 0.968)$ & 9.025 & 6.243 $(\pm 1.003)$&  9.479\\
        $\hat{\theta}^{\mathbf{sc}}$ & 4.46 $(\pm 0.958)$& 8.112 & 0.494 $(\pm 0.093)$ & 0.763\\
        $\hat{\theta}^{\mathbf{our}}$ & 3.46 $(\pm 0.689)$ & 8.676 & 0.489 $(\pm 0.093)$& 0.748\\
    \hline
    \end{tabular}
    \caption{
    The average ($\frac{1}{n}\sum_{i = 1}^n \|\hat{\theta}_i - \theta_i^*\|^2$) and maximum ($\max_i \|\hat{\theta}_i - \theta_i^*\|^2$) $\ell_2$ distance between parameter estimates and true parameters for $m \in \{10, 100\}$. The proposed method (last row) consistently outperforms existing methods in two settings.}
    \label{tab:sum_n}
\end{table}

\begin{figure}[ht]
    \centering
    \includegraphics[width=0.6\linewidth]{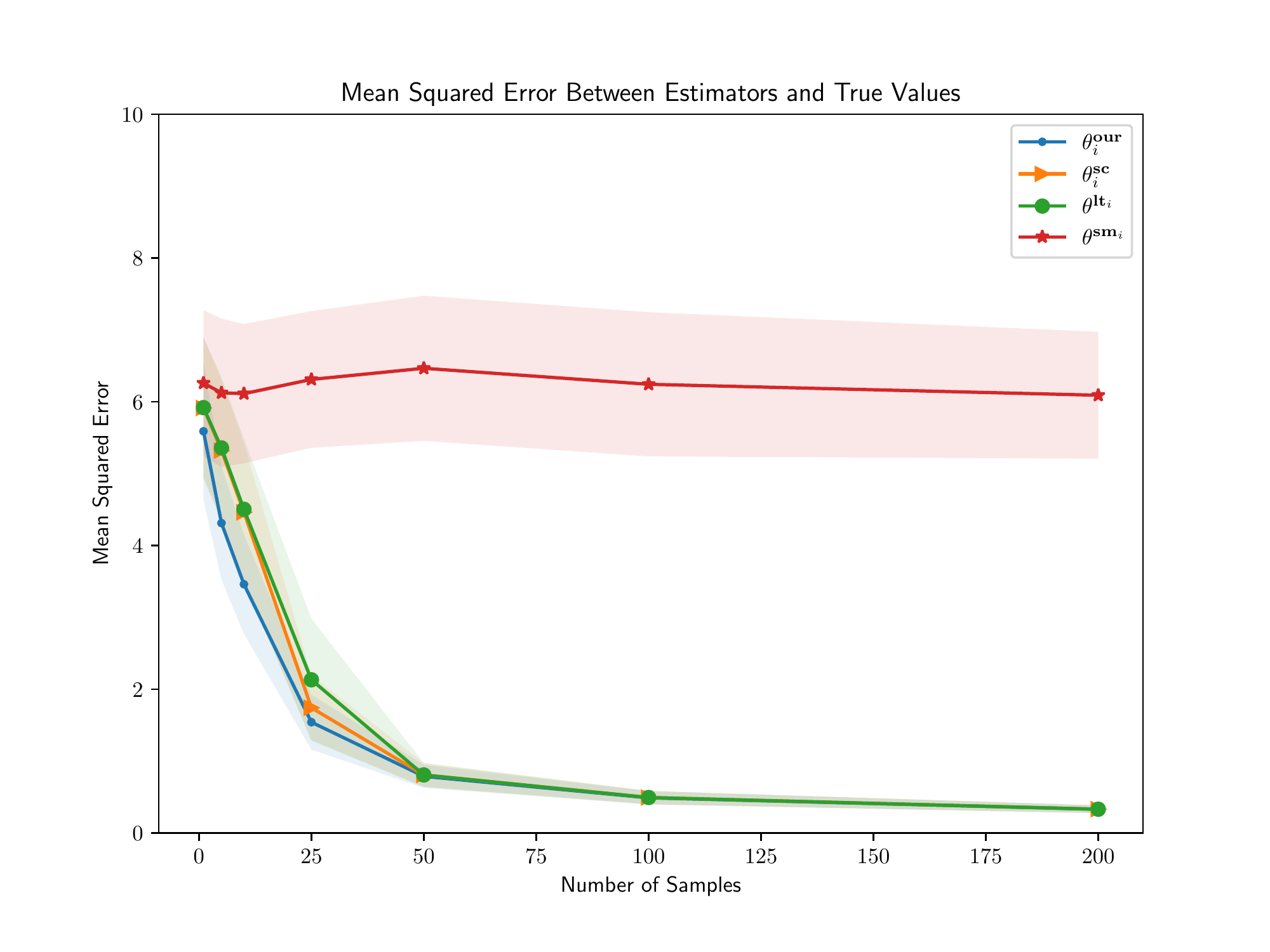}
    \caption{The average $\ell_2$ distance versus the number of local samples.}
    \label{fig:sim_mse_n}
\end{figure}

\subsection{Application: Donor Response}

We illustrate our algorithm on a real data set that contains donations and solicitation histories from a leading nonprofit organization in the US \cite{blattberg2009customer}. We follow the experimental setup described in \cite{bumbaca2017distributed}. For each solicitation record, we use its recency and frequency as covariates, where recency is defined as the number of days since the donor's last donation and frequency is the number of past donations. We define a solicitation as successful when the donor has made a donation after the current solicitation attempt and before the next solicitation attempt and model the probability of a successful solicitation using logistic regression fitted on log-transformed covariates. 

We view each ZIP code as an individual client and group the clients using the median household income of the clients in the ZIP code, based on the data obtained from \hyperlink{www.unitedstateszipcodes.org}{www.unitedstateszipcodes.org}. Using the income brackets defined in \cite{snider2019fall}, we group the ZIP code into 4 different clusters: poor-or-near-poor, lower-middle-class, middle-class, and upper-middle-class. ZIP codes with no recorded median household income are grouped into a fifth category, and ZIP codes with less than 5 solicitations are removed. We retain 29490 clients and 5 clusters after processing.

We tune the parameters $\{\lambda_j\}_{j = 1}^k, \{\gamma_i\}_{i = 1}^n$, defined in \cref{eqn:mtl_loss}, by performing a grid search over $\{10^{-2}, 10^{-1.875}, 10^{-1.750}, \ldots, 10^{1.875}, 10^2\}$. The chosen values are then used to calculate the corresponding values for $\{\alpha_j\}_{j = 1}^k$. We set the test-train ratio to 0.2, communication probability to 0.1, stepsize to $10^{-3}$, and the maximum number of iterations to 5000.

We record the accuracy and cross-entropy for each individual client. We set the cross-entropy to 100 for clients whose cross-entropy overflows. We then tune $\lambda, \gamma$ with cross-validation, based on the average cross-entropy taken across all the clients. Using cross-validation, we decide on $\lambda_j = 0.01$ for all $j$ and $\gamma_i = 1.360$ for all $j$ 

\begin{figure}[t]
    \centering
    \includegraphics[width=.3\linewidth]{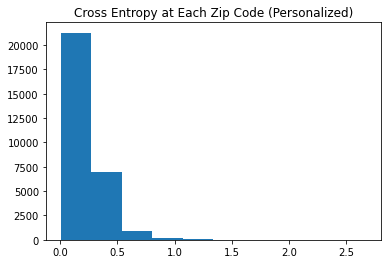}
    \includegraphics[width=.3\linewidth]{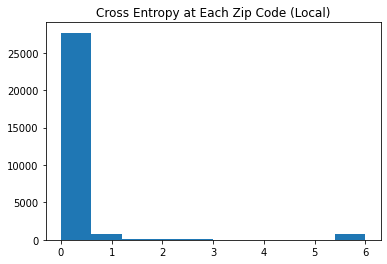}
    \includegraphics[width=.3\linewidth]{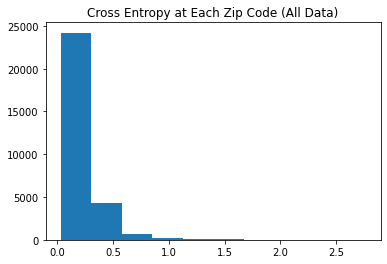}
    \caption{Distribution of cross-entropy losses at each ZIP code. From left to right: personalized model using selected parameters, models trained using entirely local data, and a single model trained on all data.}
    \label{fig:cross_hist}
\end{figure}

\Cref{fig:cross_hist} visualizes the performance of our model compared to naive baselines. Specifically, it characterizes the distribution of the clients' cross-entropy. To complement~\cref{fig:cross_hist}, we include more quantitative results in~\cref{tab:cross_tab}, which contains summary statistics on the clients' cross-entropy. From the table, we observe that our model is comparable to training a single model on all data in terms of averages and quartiles, while it has a lower maximum client-specific cross-entropy. In particular, we can see from \cref{fig:cross_hist} that our model has a smaller percentage of clients with average cross-entropy above 1.0: 183 ZIP codes have cross-entropy loss greater than 1 in the personalized model, whereas 202 have cross-entropy greater than 1 in the single-model alternative. Our model also significantly outperforms training a local model for each ZIP code separately.

\begin{table}[t]
    \centering
    \begin{tabular}{| l | c | c | c | c | c |}
        \hline
         & Avg. ($\pm$ SD.) & 25\% & 50\% & 75\% & Max\\
        \hline
        Locally-trained & $0.413(\pm 0.932)$ & 0.173 & 0.247 & 0.313 & 6 \\
        Single-model & $0.211(\pm 0.176)$ & 0.058 & 0.201 & 0.273 & 2.76 \\
        \textbf{Our Model} & $0.215(\pm0.179)$ & 0.069 & 0.209 & 0.283 & 2.668\\
        \hline
    \end{tabular}
    \caption{Summary statistics of cross-entropy losses at each ZIP code. First column is the client-specific cross entropy, second column the first quartile, third column second quartile (median), and fourth column third quartile. In the fifth column we report the maximum client-specific cross-entropy.}
    \label{tab:cross_tab}
\end{table}

\begin{table}[t]\label{tab:acc_tab}
    \centering
    \begin{tabular}{| l | c | c | c | c | c |}
        \hline
         & Avg. ($\pm$ SD.) & 25\% & 50\% & 75\% & Perf. Ratio\\
        \hline
        Locally-trained & $0.940(\pm 0.070)$ & 0.919 & 0.947 & 1 & 0.972 \\
        Single-model & $0.942(\pm 0.067)$ & 0.921 & 0.948 & 1 & 0.986 \\
        \textbf{Our Model} & $0.942(\pm 0.066)$ & 0.921 & 0.948 & 1 & 1\\
        \hline
    \end{tabular}
    \caption{Summary statistics of accuracy on each ZIP code. First column is the average client-specific accuracy, second column the first quartile, third column second quartile (median), and fourth column third quartile. We define performance ratio as the proportion of clients with accuracy no less than the model being compared to, and record it in the fifth column.}
\end{table}

Although our model attains slightly higher cross-entropy than training a single model for all clients, the accuracy of both models are almost the same, as we can see from~\cref{tab:acc_tab}. We conjecture that the higher cross-entropy is due to a lack of emphasis on tuning the learning rates and optimizing until convergence for our proposed personalized model.


\section{Conclusions and Future Directions}

We propose a new approach to personalization in federated learning when there are multiple known clusters among the clients. Our algorithm is based on a variant of loopless local gradient descent that allows each cluster to have its own communication schedule. The estimator is shown to be optimal among a class of unbiased linear estimators and performs better than commonly used alternatives. We empirically demonstrated our estimator on both simulated and real-world data.

In the future, we will investigate how to obtain an efficient implementation of alternative approaches described in~\cref{subsec:limitations_of_unbiased_estimators}. We have identified two estimators that are costly to obtain, yet outperform our proposed method. Studying their efficient implementation could yield even better methods for personalized federated learning.

\newpage
\section*{Acknowledgments}

This work is partially supported by
the William S. Fishman Faculty Research Fund
at the University of Chicago Booth School of
Business. This work was completed in part with resources provided by
the University of Chicago Research Computing Center.


{
\bibliographystyle{my-plainnat}
\bibpunct{(}{)}{,}{a}{,}{,}
\bibliography{submission}
}

\newpage

\appendix
\setcounter{theorem}{0}
\setcounter{equation}{0}
\section{Missing Lemmas and Proofs}
\label{sec:missing_lemmas_and_proofs}
We provide detailed proofs for the results in the main text. 



\subsection{Proof of \cref{prop:equiv_mtl}}
\label{subsec:proof_of_prop_equiv_mtl}
The gradient of $F(\{\theta_i\}_{i = 1}^n)$ w.r.t.~$\theta_{i_0}$ for any $i_0 \in \{1, \ldots, n\}$ is given by
\begin{equation*}
    \nabla_{\theta_{i_0}} F(\{\theta_i\}_{i = 1}^n) = \nabla f_{i_0}(\theta_{i_0})  + \alpha_{j_0}\gamma_{i_0}(\theta_{i_0} - \Bar{\theta}) + (1 - \alpha_{j_0})\gamma_{i_0}(\theta_{i_0} - \Bar{\theta}_{j_0}),
\end{equation*} where we let $j_0$ denote the cluster client $i_0$ belongs to.
The gradients of \cref{eqn:mtl_loss} w.r.t.~$\{w_j\}_{j = 1}^k, \Bar{w}$ are given by
\begin{align*}
    &\nabla_{w_j}F_{MTL}(\{\theta_i\}_{i = 1}^n, \{w_j\}_{j = 1}^k, \Bar{w}) = \sum_{i \in \cI_j}(\theta_i - w_j) + \lambda_j(w_j - \Bar{w}), \,j = 1, \ldots, k,\\
    &\nabla_{\Bar{w}}F_{MTL}(\{\theta_i\}_{i = 1}^n, \{w_j\}_{j = 1}^k, \Bar{w}) = \sum_{j = 1}^k \lambda_j(w_j - \Bar{w}).
\end{align*} 
Setting the gradients to zero and solving the resulting linear system, give us
\begin{align*}
    &w_j = \frac{\lambda_j}{\lambda_j + \sum_{i \in \cI_j} \gamma_i}\Bar{w} + \frac{\sum_{i \in \cI_j} \gamma_i\theta_i}{\lambda_j + \sum_{i \in\cI_j} \gamma_i}, \qquad j = 1, \ldots, k,\\
    &\Bar{w} = \rbr{\sum_{j = 1}^k \frac{\lambda_j \sum_{i \in \cI_j} \gamma_i}{\lambda_j + \sum_{i \in \cI_j}\gamma_i}}^{-1}\rbr{\sum_{j = 1}^k\frac{\lambda_j \sum_{i \in \cI_j}\gamma_i\theta_i}{\lambda_j + \sum_{i \in \cI_j}\gamma_i}}.
\end{align*} 
Plugging $\alpha_j = \frac{\lambda_j}{\lambda_j + \sum_{i \in \cI_j} \gamma_i}$ and recalling \cref{eqn:theta_bar_defn}, we note that 
\begin{equation*}
    \Bar{w} = \frac{\sum_{j = 1}^k \sum_{i \in \cI_j} \alpha_j\gamma_i\theta_i}{\sum_{j = 1}^k \sum_{i \in \cI_j}\alpha_j\gamma_i} = \Bar{\theta},
\end{equation*} 
and
\begin{equation*}
    w_j = \alpha_j \Bar{\theta} + (1 - \alpha_j) \Bar{\theta}_j, \qquad j = 1, \ldots, k.
\end{equation*} 
Therefore, the first order condition for any $\theta_{i_0}$ is
\begin{equation*}
    \nabla f_{i_0}(\theta_{i_0}) + \alpha_{j_0}\gamma_{i_0}(\theta_{i_0} - \Bar{\theta}) + (1 - \alpha_{j_0})\gamma_{i_0}(\theta_{i_0} - \Bar{\theta}_{j_0}) = 0,
\end{equation*} 
which is exactly the first order condition $\nabla_{\theta_i} F(\{\theta_i\}_{i = 1}^n) = 0$. The proof is now complete.

\subsection{Proof of \cref{prop:reg_grad}} 
\label{subsec:proof_of_prop_reg_grad}

We prove the first statement for some $j$ and $i \in \cI_j$. Since
    \begin{equation*}
        \nabla_{\theta_i} \Bar{\theta}_j = \frac{\gamma_i}{\sum_{i' \in \cI_j}\gamma_{i'}},
    \end{equation*}
    we have
    \begin{align*}
        \nabla_{\theta_i}\psi_j\rbr{\btheta_j; \cbr{\gamma_i}_{i \in \cI_j}} &= \gamma_i\rbr{\theta_i - \Bar{\theta}_j}\rbr{1 - \frac{\gamma_i}{\sum_{i' \in \cI_j}\gamma_{i'}}}\\
        &\quad + \sum_{i' \in \cI_j\backslash\{i\}}\gamma_{i'}\rbr{\theta_{i'} - \Bar{\theta}_j}\rbr{-\frac{\gamma_i}{\sum_{i' \neq \cI_j}\gamma_{i'}}}\\
        &= \gamma_i\rbr{\theta_i - \Bar{\theta}_j} - \sum_{i' \in \cI_j}\frac{\gamma_{i'}}{\sum_{i'' \in \cI_j}\gamma_{i''}}\rbr{\theta_{i'} - \Bar{\theta}_j}\\
        &= \gamma_i(\theta_i - \Bar{\theta}_j).
    \end{align*}
    The second statement is proven similarly.

\subsection{Proof of \cref{prop:opt_tau}}
\label{subsec:proof_of_prop_opt_tau}

Note that the gradient oracle, $G(\btheta)$, can be written as
    \begin{align*}
        G(\btheta) &= \frac{\ind\{\xi_0 = 1\}}{p_0}\rbr{\nabla_{\btheta}\varphi(\btheta) + \begin{bmatrix}
            (1 - \alpha_1)\tau_1 \nabla_{\btheta_1}\psi_i(\btheta_1)\\
            (1 - \alpha_2)\tau_2 \nabla_{\btheta_2}\psi_2(\btheta_2)\\
            \vdots\\
            (1 - \alpha_k)\tau_k \nabla_{\btheta_k}\psi_i(\btheta_k)
        \end{bmatrix}} + \\
        &\quad \quad \frac{\ind\{\xi_0 = 0\}}{1 - p_0}\begin{bmatrix}
            \frac{\ind\{\xi_1 = 1\}}{p_1}(1 - \alpha_1)(1 - \tau_1)\nabla_{\btheta_1}\psi_1(\btheta_1) + \frac{\ind\{\xi_1 = 0\}}{1 - p_1}\nabla_{\btheta_1}F_1(\btheta_1)\\
            \frac{\ind\{\xi_2 = 1\}}{p_2}(1 - \alpha_2)(1 - \tau_2)\nabla_{\btheta_2}\psi_2(\btheta_2) + \frac{\ind\{\xi_2 = 0\}}{1 - p_2}\nabla_{\btheta_2}F_2(\btheta_2)\\
            \vdots\\
            \frac{\ind\{\xi_k = 1\}}{p_k}(1 - \alpha_k)(1 - \tau_k)\nabla_{\btheta_k}\psi_k(\btheta_k) + \frac{\ind\{\xi_k = 0\}}{1 - p_k}\nabla_{\btheta_k}F_k(\btheta_k)
        \end{bmatrix}.
    \end{align*}
    We bound the two terms separately using the law of total expectation. Note that
    \begin{align*}
        \EE&\sbr{\|G(\btheta) - G(\btheta^*)\|^2 | \xi_0 = 1}\PP(\xi_0 = 1) \\
        &=\frac{1}{p_0}\nbr{\nabla_{\btheta}\varphi(\btheta) - \nabla_{\btheta}\varphi(\btheta^*) + \begin{bmatrix}
            (1 - \alpha_1)\tau_1 (\nabla_{\btheta_1}\psi_i(\btheta_1) - \nabla_{\btheta_1}\psi_1(\btheta_1^*))\\
            (1 - \alpha_2)\tau_2 (\nabla_{\btheta_2}\psi_2(\btheta_2) - \nabla_{\btheta_2}\psi_2(\btheta_2^*))\\
            \vdots\\
            (1 - \alpha_k)\tau_k (\nabla_{\btheta_k}\psi_i(\btheta_k) - \nabla_{\btheta_k}\psi_k(\btheta_k^*))
        \end{bmatrix}}^2\\
        & \leq \frac{2}{p_0} \nbr{\nabla_{\btheta}\varphi(\btheta) - \nabla_{\btheta}\varphi(\btheta^*)}^2 + \frac{2}{p_0}\nbr{\begin{bmatrix}
            (1 - \alpha_1)\tau_1 (\nabla_{\btheta_1}\psi_i(\btheta_1) - \nabla_{\btheta_1}\psi_1(\btheta_1^*))\\
            (1 - \alpha_2)\tau_2 (\nabla_{\btheta_2}\psi_2(\btheta_2) - \nabla_{\btheta_2}\psi_2(\btheta_2^*))\\
            \vdots\\
            (1 - \alpha_k)\tau_k (\nabla_{\btheta_k}\psi_i(\btheta_k) - \nabla_{\btheta_k}\psi_k(\btheta_k^*))
        \end{bmatrix}}^2\\
        & = \frac{2}{p_0}\nbr{\nabla_{\btheta}\varphi(\btheta) - \nabla_{\btheta}\varphi(\btheta^*)}^2 + \frac{2}{p_0}\sum_{j = 1}^k (1 - \alpha_j)^2\tau_j^2\|\nabla_{\btheta_1}\psi_i(\btheta_1) - \nabla_{\btheta_1}\psi_1(\btheta_1^*)\|^2,
    \end{align*}
    where the second line uses the fact that $\EE$ is taken over $\{\xi_j\}_{j = 0}^k$, the third line uses the triangle inequality, and the fourth line follows directly using the fact that $\|\cdot\|^2$ is the squared $\ell_2$-norm.
    
    Similarly, we have
    \begin{align*}
        &\EE\sbr{\|G(\btheta) - G(\btheta^*)\|^2 | \xi_0 = 0}\PP(\xi_0 = 0)
        = \frac{1}{1 - p_0}\sum_{j = 1}^k\EE_{\xi_j}{M_j},
    \end{align*} 
    where
    \begin{align*}
        M_j = &\Biggl\|\frac{\ind\{\xi_j = 1\}}{p_j}(1 - \alpha_j)(1 - \tau_j)(\nabla_{\btheta_j}\psi_j(\btheta_j) - \nabla_{\btheta_j}\psi_j(\btheta_j^*)) \\
        &\hspace{8em}+ \frac{\ind\{\xi_j = 0\}}{1 - p_j}(\nabla_{\btheta_j}F_j(\btheta_j) - \nabla_{\btheta_j}F_j(\btheta_j^*))\Biggr\|^2.
    \end{align*}
    For a fixed $j$, we have
    \begin{align*}
        \EE_{\xi_j}[M_j]
        & = \EE_{\xi_j}[M_j | \xi_j = 1]\PP(\xi_j = 1) + \EE_{\xi_j}[M_j | \xi_j = 0]\PP(\xi_j = 0)\\
        & = \frac{1}{p_j}(1 - \alpha_j)^2(1 - \tau_j)^2\|\nabla_{\btheta_j}\psi_j(\btheta_j) - \nabla_{\btheta_j}\psi_j(\btheta_j^*)\|^2\\ 
        & \qquad + \frac{1}{1 - p_j}\|\nabla_{\btheta_j}F_j(\btheta_j) - \nabla_{\btheta_j}F_j(\btheta_j^*)\|^2.
    \end{align*}
    Combining the two bounds, we further have
    \begin{align*}
        \EE&[\|G(\btheta) - G(\btheta^*)\|^2]\\
        &=\EE\sbr{\|G(\btheta) - G(\btheta^*)\|^2 | \xi_0 = 1}\PP(\xi_0 = 1) + \EE\sbr{\|G(\btheta) - G(\btheta^*)\|^2 | \xi_0 = 0}\PP(\xi_0 = 0)\\
        &\leq \frac{2}{p_0} \nbr{\nabla_{\btheta}\varphi(\btheta) - \nabla_{\btheta}\varphi(\btheta^*)}^2 \\
        &\qquad + \sum_{j = 1}^k (1 - \alpha_j)^2\rbr{\frac{2}{p_0}\tau_j^2 + \frac{(1 - \tau_j)^2}{p_j(1 - p_0)}}\|\nabla_{\btheta_j}\psi_j(\btheta_j) - \nabla_{\btheta_j}\psi_j(\btheta_j^*)\|^2 \\
        &\qquad + \sum_{j = 1}^k \frac{1}{1 - p_j}\|\nabla_{\btheta_j}F_j(\btheta_j) - \nabla_{\btheta_j}F_j(\btheta_j^*)\|^2.
    \end{align*}
    In the above display, only the second term depends on $\tau_j$. 
    This term is minimized when 
    \[
    \tau_j = \frac{p_0}{p_0 + 2(1 - p_0)p_j}.
    \]
    The result immediately follows now.

\subsection{Proof of \cref{prop:reg_conv}}
\label{subsec:proof_of_prop_reg_conv}

We prove the first statement for a fixed $j$. From the proof of \cref{prop:reg_grad}, we can write the gradient of $\psi_j\rbr{\btheta_j; \cbr{\gamma_i}_{i \in \cI_j}}$ w.r.t.~$\btheta_j$ as
\begin{equation*}  
    \nabla_{\btheta_j}\psi_j\rbr{\btheta_j; \cbr{\gamma_i}_{i \in \cI_j}} = \btheta_j - \begin{bmatrix}
        \Bar{\theta}_j\\
        \Bar{\theta}_j\\
        \vdots\\
        \Bar{\theta}_j
    \end{bmatrix}
\end{equation*} 
and the Hessian as $\nabla_{\btheta_j\btheta_j}\psi_j\rbr{\btheta_j; \cbr{\gamma_i}_{i \in \cI_j}} = I_d \otimes H_j$,
where $H_j \in \RR^{|\cI_j| \times |\cI_j|}$ is a symmetric matrix with element in row $i_1$ and column $i_2$ ($i_1, i_2 \in \cI_j$) denoted as 
\begin{equation*}
    H_j[i_1, i_2] = \begin{cases}
        \gamma_{i_1}\rbr{1 - \frac{\gamma_{i_1}}{\sum_{i \in \cI_j}\gamma_i}}, \quad &\text{if $i_1 = i_2$},\\
        -\frac{\gamma_{i_1}\gamma_{i_2}}{\sum_{i = 1 \in \cI_j}\gamma_i}, \quad &\text{otherwise}.
    \end{cases}
\end{equation*}
For an arbitrary fixed vector $v \in \RR^{|\cI_j|}$, we have
\begin{multline*}
    v^TH_jv = \sum_{l = 1}^{|\cI_j|}\gamma_{\cI_j[l]} (v[l])^2 - \frac{1}{\sum_{l = 1}^{|\cI_j|}\gamma_{\cI_j[l]}}\rbr{\sum_{l = 1}^{|\cI_j|}\gamma_{\cI_j[l]} (v[l])}^2\\
    = \rbr{\frac{1}{\sum_{l = 1}^{|\cI_j|}\gamma_{\cI_j[l]}}}\rbr{\sum_{l = 1}^{|\cI_j|}\frac{\gamma_{\cI_j[l]}}{\sum_{l' = 1}^{|\cI_j|}\gamma_{\cI_j[l']}} (v[l])^2 - \rbr{\sum_{l = 1}^{|\cI_j|}\frac{\gamma_{\cI_j[l]}}{\sum_{l' = 1}^{|\cI_j|}\gamma_{\cI_j[l']}} v[l]}^2}
    \geq 0,
\end{multline*}
which shows that $H_j \succeq 0$.
Let $\Gamma_j \in \RR^{|\cI_j| \times |\cI_J|}$ be a diagonal matrix with diagonal elements $\{\gamma_{\cI_j[l]}\}_{l = 1}^{|\cI_j|}$. Then
\begin{align*}
    v^T\Gamma_j v - v^TH_j v = \frac{1}{\sum_{l = 1}^{|\cI_j|}\gamma_{\cI_j[l]}}\rbr{\sum_{l = 1}^{|\cI_j|}\gamma_{\cI_j[l]} (v[l])}^2
    \geq 0,
\end{align*} 
which shows that $H_j \preceq \Gamma_j \preceq \max_{i \in \cI_j} \gamma_i I_{|\cI_J|}$. The first statement now follows, since the eigenvalues of the Kronecker product of two matrices are the products of the pairs of eigenvalues of two matrices.

The second statement is established in the same way.

\subsection{Proof of \cref{thm:conv_rate}}
\label{subsec:proof_of_thm_conv_rate}

First, we show that the gradient estimator defined in \cref{eqn:gradient_oracle} satisfies the expected smoothness condition. We then bound the gradient estimator's second moment, and apply Theorem 3.1 from \cite{gower2019sgd} to complete the proof.

\begin{lemma}[Expected Smoothness]
    \label{lemma:es}
    Suppose conditions of \cref{thm:conv_rate} hold and $\cL$
    is defined in \cref{thm:conv_rate}.
    Then 
    \begin{equation*}
        \EE[\|G(\btheta) - G(\btheta^*)\|^2] \leq 2\cL \rbr{F(\btheta) - F(\btheta^*)},
    \end{equation*} 
    where the expectation is taken over the randomness in $\{\xi_j\}_{j= 0}^k$.
\end{lemma}
\begin{proof}
For a convex function $g: \RR^d \to \RR$, 
let $D_g(x, y) = g(x) - g(y) - (\nabla g(y))^T(x - y)$, $x, y \in \RR^d$.
By \cref{prop:reg_conv} and \cref{thm:breg_smooth}, we have
\begin{align*}
    \nbr{\nabla_{\btheta}\varphi(\btheta) - \nabla_{\btheta}\varphi(\btheta^*)}^2 
    &\leq 2\max_{j = 1, \ldots, k}\max_{i \in \cI_j}\alpha_j\gamma_i D_{\varphi}(\btheta, \btheta^*),\\
    \nbr{\nabla_{\btheta_j}\psi_j(\btheta_j) - \nabla_{\btheta_j}\psi_j(\btheta_j^*)}^2 
    & \leq 2 \max_{i \in \cI_j}\gamma_i D_{\psi_j}(\btheta_j, \btheta_j^*),
\end{align*}
while \cref{assumption:smooth_and_convex} states that
\begin{equation*}
        \|\nabla_{\btheta_j}F_j(\btheta_j) - \nabla_{\btheta_j}F_j(\btheta_j^*)\|^2 \leq 2L D_{F_j}(\btheta_j, \btheta^*_j).
\end{equation*}
Plugging into the result of \cref{prop:opt_tau}, we have
\begin{align*}
    \EE[\|G(\btheta) - G(\btheta^*)\|^2]
    & \leq \frac{4}{p_0}\max_{j = 1, \ldots, k}\max_{i \in \cI_j}\alpha_j\gamma_i D_{\varphi}(\btheta, \btheta^*) \\
    & \qquad + \sum_{j = 1}^k \frac{4(1 - \alpha_j)^2\max_{i \in \cI_j}\gamma_i}{p_0 + 2(1 - p_0)p_j}D_{\psi_j}(\btheta_j, \btheta_j^*)  \\
    & \qquad + \frac{2L}{1 - p_0}\sum_{j = 1}^k \frac{1}{1 - p_j}D_{F_j}(\btheta_j, \btheta^*_j).
\end{align*}
Since
\begin{multline*}
F(\btheta) - F(\btheta^*) 
= D_{F}(\btheta, \btheta^*) \\
= D_{\varphi}(\btheta, \btheta^*) + \sum_{j = 1}^k (1 - \alpha_j)D_{\psi_j}(\btheta_j, \btheta_j^*) + \sum_{j = 1}^k D_{F_j}(\btheta_j, \btheta_j^*),
\end{multline*} 
we have 
\begin{align*}
    \EE & [\|G(\btheta) - G(\btheta^*)\|^2]\\
    & \leq 2\cL \rbr{D_{\varphi}(\btheta, \btheta^*) + \sum_{j = 1}^k (1 - \alpha_j)D_{\psi_j}(\btheta_j, \btheta_j^*) + \sum_{j = 1}^k D_{F_j}(\btheta_j, \btheta_j^*)}\\
    & =  2\cL(F(\btheta) - F(\btheta^*)),
    \end{align*}
which completes the proof.
\end{proof}

\begin{corollary}[Bounded Second Moment]
    \label{corollary:sm}
    Suppose conditions of \cref{thm:conv_rate} hold and $\cL$
    and $\sigma_{\btheta^*}^2$ are defined in \cref{thm:conv_rate}.    
    Then
    \begin{equation*}
        \EE[\|G(\btheta)\|^2] \leq 4\cL\EE[(F(\btheta) - F(\btheta^*)] + 2\sigma_{\btheta^*}^2,
    \end{equation*} 
    where the expectation is taken over the randomness in $\{\xi_j\}_{j= 0}^k$.
\end{corollary}
\begin{proof}
From the proof of \cref{prop:opt_tau}, we have $\EE[\|G(\btheta^*)\|^2] \leq \sigma_{\btheta^*}^2$. The result follows from Lemma 2.4 of \cite{gower2019sgd}.
\end{proof}

\subsection{Proof of \cref{corollary:opt_conv_rate}}
\label{subsec:proof_of_corollary_opt_conv_rate}

When $C_2 > C_1$, $\Tilde{\cL}$ is minimized when $p_0 = \frac{2C_1}{C_1 + C_2 + L}$ and $p_j = \frac{C_2 - C_1}{C_2 - C_1 + L}$. When $C_2 \leq C_1$, $\Tilde{\cL}$ is minimized when $p_0 = \frac{2C_1}{2C_1 + L}$ and $p_j = 0$ for all $j$. Then
\begin{equation*}
    \cL \leq \Tilde{\cL} = \begin{cases}
        C_1 + C_2 + L &\text{ if $C_2 > C_1$},\\
        2C_1 + L &\text{ if $C_2 \leq C_1$}.
    \end{cases}
\end{equation*}

With the upperbounds on $\cL$ determined for when $C_2 > C_1$ and when $C_2 \leq C_1$, we are left with deriving the optimal $\tau_j$ under the two settings. For the choices of $\{p_j\}_{j = 0}^k$ we have derived for the two settings, by~\cref{prop:opt_tau}, we know that when $C_2 > C_1$ we should set $\tau_j = \frac{C_1}{C_2}$ for all $j$ and when $C_2 \leq C_1$ we set $\tau_j = 1$ for all $j$.

For any $\epsilon > 0$, set $\eta = (2\cL)^{-1}$ and $t = {2\cL}/{\mu}\log ({1}/{\epsilon})$. Then
\[
    (1 - \eta\mu)^t \leq \exp\cbr{-t\eta\mu} = \epsilon.
\]
By \cref{prop:expected_num_comm}, we have the following.
\begin{enumerate}
    \item When $C_2 > C_1$, $\cL \leq C_1 + C_2 + L$ and $t \leq \frac{2(C_1 + C_2 + L)}{\mu}\log \frac{1}{\epsilon}$. Since $p_0 = \frac{2C_1}{C_1 + C_2 + L}$, $p_j = \frac{C_2 - C_1}{C_2 - C_1 + L}$, the expected number of between-cluster communication is at most $\frac{4C_1(C_2 - C_1 + L)}{(C_1 + C_2 + L)\mu}\log \frac{1}{\epsilon}$, and the expected number of within-cluster communication is at most $\frac{2L(C_2 - C_1)}{(C_2 - C_1 + L)\mu}\log\frac{1}{\epsilon}$.
    \item When $C_2 \leq C_1$, we have $\cL \leq 2C_1 + L$, therefore $t \leq \frac{2(2C_1 + L)}{\mu}$. Recalling the choices for $p_0$, the number of between-cluster communication is at most $\frac{4C_1 L}{(2C_1 + L)\mu}\log \frac{1}{\epsilon}$.
\end{enumerate}

\subsection{Proof of \cref{lemma:al2sgd_gradient_oracle_expected_smoothness}}
\label{subsec:proof_of_lemma_al2sgd_gradient_oracle_expected_smoothness}

Similarly to the proof of \cref{lemma:es}, we start by conditioning $\EE[\|\bg^t - \nabla F(x^t)\|^2]$ on $\xi_0 = 1$ and $\xi_0 = 0$, respectively. For $\EE[\|\bg^t - \nabla F(x^t)\|^2 \mid \xi_0 = 0]$, we further expand $\EE[\|g^t_i - \nabla_i F(x^t)\|^2]$ and condition on $\xi_j = 1$ and $\xi_j = 0$, $j = 1, \ldots, k$, $i \in \cI_j$. These conditional expectations can then be bounded by Bregman divergences, which completes the proof.

\subsection{Proof of \cref{thm:al2sgd_conv_rate}}
\label{subsec:proof_of_thm_al2sgd_conv_rate}

Similarly to the proof of \cref{prop:reg_conv}, we can show that $\sum_{j= 1}^k \psi_j\left(\btheta_j; \cbr{\gamma_i}_{i \in \cI_j}\right) + \varphi\rbr{\btheta; \cbr{\gamma_i}_{i = 1}^n, \cbr{\alpha}_{j = 1}^k}$ is $\max_{i = 1, \ldots, n}\gamma_i$-smooth and convex in $\btheta$. The function $F(\btheta)$ is $\Tilde{L} + \max_{i = 1, \ldots, n}\gamma_i$-smooth and $\mu$-strongly convex in $\btheta$ under \cref{assumption:strong_conv_smooth_ae}.
    
\cref{alg:async_al2sgd_plus} is a special instance of \texttt{L-Katyusha}~\cite{hanzely2020variance}. For each $t = 1, 2, \ldots,$ we obtain an unbiased stochastic gradient estimate $g_i^t$ for each client. The local updates of the clients follow the form of the updates in \texttt{L-Katyusha}. The random variable $\xi'$ then controls how often the algorithm updates the full gradient. Plugging the expected smoothness of the stochastic gradient oracle, given in \cref{lemma:al2sgd_gradient_oracle_expected_smoothness}, into Theorem 4.1 from \cite{hanzely2020variance} completes the proof.

\subsection{Proof of \cref{corollary:async_al2sgd_plus_opt_conv_rate}} 
\label{subsec:proof_of_corollary_asyn_al2sgd_plus_opt_conv_rate}

The proof is similar to the proof of \cref{corollary:opt_conv_rate} and is omitted.

\subsection{Proof of \cref{thm:blue}} 
\label{subsec:proof_of_thm_blue}

Under our model, we have
\[
    \hat{\theta}_i^d = \theta_i^* + \frac{1}{\beta_i}X_i^T\epsilon_i = \Bar{\theta}_j^* + \xi_i + \frac{1}{\beta_i}X_i^T\epsilon_i = \Bar{\theta}^* + \Bar{\xi}_j + \xi_i + \frac{1}{\beta_i}X_i^T\epsilon_i.
\]
By a direct calculation, solution to \cref{eqn:hlm_loss} can be written as
\begin{align*}
    \hat{\theta}_i &= \frac{\beta_i/\sigma_i^2}{\beta_i/\sigma_i^2 + \gamma_i}\hat{\theta}_i^d + \frac{\gamma_i}{\beta_i/\sigma_i^2 + \gamma_i}\hat{w}_j, \\
    \hat{w}_j &= \frac{\sum_{i \in \cI_j} \gamma_i\hat{\theta}_i}{\sum_{i \in \cI_j}\gamma_i + \lambda_j} + \frac{\lambda_j}{\sum_{i \in \cI_j}\gamma_i + \lambda_j}\hat{\Bar{w}}, \\
    \hat{\Bar{w}} &= \frac{1}{\sum_{j = 1}^k\lambda_j}\sum_{j = 1}^k \lambda_j \hat{w}_j. 
\end{align*}
Expanding the right hand side of $\hat{w}_j$, we have
\begin{align*}
    \hat{w}_j &= \frac{1}{\sum_{i \in \cI_j}\gamma_i + \lambda_j}\left(\sum_{i \in \cI_j} \frac{\gamma_i \beta_i/\sigma_i^2}{\beta_i/\sigma_i^2 + \gamma_i}\hat{\theta}_i^d + \sum_{i \in \cI_j}\frac{\gamma_i^2}{\beta_i/\sigma_i^2 + \gamma_i}\hat{w}_j\right) + \frac{\lambda_j}{\sum_{i \in \cI_j}\gamma_i + \lambda_j}\hat{\Bar{w}}.
\end{align*} 
Let 
\[
C_i = \frac{\gamma_i \beta_i/\sigma_i^2}{\beta_i/\sigma_i^2 + \gamma_i}
\qquad\text{and}\qquad
\hat{\Bar{\theta}}_j = \left(\sum_{i \in \cI_j} C_i\right)^{-1}\left(\sum_{i \in \cI_j}C_i\hat{\theta}_i^d\right).
\]
With this notation, we have
\begin{align*}
    \hat{w}_j &= \frac{\sum_{i \in \cI_j}C_i}{\lambda_j + \sum_{i \in \cI_j} C_i}\hat{\Bar{\theta}}_j + \frac{\lambda_j}{\lambda_j + \sum_{i \in \cI_j} C_i}\hat{\Bar{w}} \\
\intertext{and}
    \hat{\Bar{w}} &=\frac{1}{\sum_{j = 1}^k\lambda_j}\sum_{j = 1}^k \lambda_j \left( \frac{\sum_{i \in \cI_j}C_i}{\lambda_j + \sum_{i \in \cI_j} C_i}\hat{\Bar{\theta}}_j + \frac{\lambda_j}{\lambda_j + \sum_{i \in \cI_j} C_i}\hat{\Bar{w}}\right).
\end{align*}
Let $$D_j = \frac{\lambda_j \sum_{i \in \cI_j}C_i}{\lambda_j + \sum_{i \in \cI_j} C_i}.$$
Then
\[
    \hat{\Bar{w}} = \left(\sum_{j = 1}^k D_j\right)^{-1}\sum_{j = 1}^k D_j\hat{\Bar{\theta}}_j.
\]
Without loss of generality, 
$\cI_1 = \{1, \dots, |\cI_1|\}$, 
$\cI_2 = \{|\cI_1| + 1, |\cI_1| + 2, \dots, |\cI_1| + |\cI_2|\}$,
and so on. For all $i' \in \cI_1$, $i' \neq 1$, we have
\[
    \hat{\theta}_{i'}^d = \theta_1^* - \xi_1 + \xi_{i'} + \frac{1}{\beta_{i'}}X_{i'}^T\epsilon_{i'}.
\] 
For all $i' \in \cI_{j'}, j' \neq 1$, we have
\[
    \hat{\theta}_{i'}^d = \theta_1^* - \xi_1 - \Bar{\xi}_1 + \Bar{\xi}_{j'} + \xi_{i'} + \frac{1}{\beta_{i'}}X_{i'}^T\epsilon_{i'}.
\]
Then
\begin{equation}
    \label{eqn:y_theta_distribution}
    \begin{bmatrix}
        y_1 \\
        \hat{\theta}_2^d \\
        \vdots\\
        \hat{\theta}_m^d
    \end{bmatrix} = \begin{bmatrix}
        X_1\\
        I\\
        \vdots\\
        I
    \end{bmatrix}\theta_1^* + \underbrace{\begin{bmatrix}
        \epsilon_1\\
        - \xi_1 + \xi_2 + \frac{1}{\beta_2}X_2^T\epsilon_2\\
        \vdots\\
        - \xi_1 - \Bar{\xi}_1 + \Bar{\xi}_k + \xi_n + \frac{1}{\beta_n}X_n^T\epsilon_n.
    \end{bmatrix}}_{\bm{\zeta}_1},
\end{equation}
where
$$
\bm{\zeta}_1 \sim N\left(0, \begin{bmatrix}
    \sigma_1^2 & \bm{0}_{n - 1}^T\\
    \bm{0}_{n - 1} & \Omega_1
\end{bmatrix}\otimes I_d\right)
$$
and for any $i, i' \in \{1, \ldots, n\}$,
\[
    \Omega_1[i, i'] = \begin{cases}
        2\Bar{\sigma}_1^2 + \sigma_i^2/\beta_i \quad &\text{if $i = i' \in \cI_1$}\\
        \Bar{\sigma}_1^2 + \Bar{\sigma}_j^2 + 2\Bar{\sigma}^2 + \sigma_i^2/\beta_i &\text{ if $i = i' \in \cI_j$, $j \neq 1$}\\
        \Bar{\sigma}_1^2 &\text{ if $i \in \cI_1$ or $i' \in \cI_1$, $i \neq i'$}\\
        \Bar{\sigma}_1^2 + \Bar{\sigma}^2 &\text{ if $i \in \cI_j, i' \not\in \cI_j$, $i, i' \not\in \cI_1$, $i \neq i'$}\\
        \Bar{\sigma}_1^2 + 2\Bar{\sigma}^2 &\text{ if $i, i' \in \cI_j, j \neq 1$, $i \neq i'$}
    \end{cases}.
\] 
The matrix $\Omega_1$ can be expressed as 
\[
    \Omega_1 = \begin{bmatrix}
        \diag(\{\sigma_i^2/\beta_i + \Bar{\sigma}_1^2\}_{i = 2}^{|\cI_1|}) & \bm{0}_{|\cI_1| - 1}\bm{0}_{n - |\cI_1|}^T\\
        \bm{0}_{n - |\cI_1|}\bm{0}_{|\cI_1| - 1}^T & \Omega_{-1}
    \end{bmatrix} + \Bar{\sigma}_1^2 \mathbf{1}_{n - 1}\mathbf{1}_{n - 1}^T,
\] where
\[
    \Omega_{-1} = \begin{bmatrix}
        \Omega_{-1}^{(2)} & \bm{0}_{|\cI_2|}\bm{0}_{|\cI_3|}^T & \dots &\bm{0}_{|\cI_2|}\bm{0}_{|\cI_k|}^T\\
        \bm{0}_{|\cI_3|}\bm{0}_{|\cI_2|}^T &\Omega_{-1}^{(3)} & \dots & \bm{0}_{|\cI_3|}\bm{0}_{|\cI_k|}^T\\
        \vdots & \vdots & \vdots & \vdots\\
        \bm{0}_{|\cI_k|}\bm{0}_{|\cI_2|}^T &\bm{0}_{|\cI_k|}\bm{0}_{|\cI_3|}^T &\dots &\Omega_{-1}^{(k)}
    \end{bmatrix} + \Bar{\sigma}^2\mathbf{1}_{n - |\cI_1|}\mathbf{1}_{n - |\cI_1|}^T,
\] 
and 
$$\Omega_{-1}^{(j)} = \diag(\{\sigma_i^2/\beta_i + \Bar{\sigma}_j^2\}_{i \in \cI_j}) + \Bar{\sigma}^2\mathbf{1}_{|\cI_j|}\mathbf{1}_{|\cI_j|}^T,
\qquad j \neq 1.$$ 
By~Woodbury~matrix~identity~\citep{petersen2008matrix},
\begin{multline*}
    \left(\Omega_{-1}^{(j)}\right)^{-1} = \diag(\{(\sigma_i^2/\beta_i + \Bar{\sigma}_j^2)^{-1}\}_{i \in \cI_j}) - 
    \left((\Bar{\sigma}^2)^{-1} + \sum_{i \in \cI_j}(\sigma_i^2/\beta_i + \Bar{\sigma}_j^2)^{-1}\right)^{-1} \\
    \times
    \begin{bmatrix}
        (\sigma_{\cI_j[1]}^2/\beta_{\cI_j[1]} + \Bar{\sigma}_j^2)^{-1}\\
        (\sigma_{\cI_j[2]}^2/\beta_{\cI_j[2]} + \Bar{\sigma}_j^2)^{-1}\\
        \vdots\\
        (\sigma_{\cI_j[|\cI_j|]}^2/\beta_{\cI_j[|\cI_j|]} + \Bar{\sigma}_j^2)^{-1}
    \end{bmatrix}\begin{bmatrix}
        (\sigma_{\cI_j[1]}^2/\beta_{\cI_j[1]} + \Bar{\sigma}_j^2)^{-1}\\
        (\sigma_{\cI_j[2]}^2/\beta_{\cI_j[2]} + \Bar{\sigma}_j^2)^{-1}\\
        \vdots\\
        (\sigma_{\cI_j[|\cI_j|]}^2/\beta_{\cI_j[|\cI_j|]} + \Bar{\sigma}_j^2)^{-1}
    \end{bmatrix}^T. 
\end{multline*}
With this, we have
\[
    \mathbf{1}_{|\cI_j|}^T\left(\Omega_{-1}^{(j)}\right)^{-1}\mathbf{1}_{|\cI_j|} = \left(\left(\sum_{i' \in \cI_j} \frac{1}{\sigma_i^2/\beta_i + \Bar{\sigma}_j^2}\right)^{-1} + \Bar{\sigma}^2\right)^{-1},
\] 
\[
    \left(\Omega_{-1}^{(j)}\right)^{-1}\mathbf{1}_{|\cI_j|} = \frac{\left(\sum_{i' \in \cI_j} \frac{1}{\sigma_i^2/\beta_i + \Bar{\sigma}_j^2}\right)^{-1}}{\left(\sum_{i' \in \cI_j} \frac{1}{\sigma_i^2/\beta_i + \Bar{\sigma}_j^2}\right)^{-1} + \Bar{\sigma}^2} \begin{bmatrix}
        (\sigma_{\cI_j[1]}^2/\beta_{\cI_j[1]} + \Bar{\sigma}_j^2)^{-1}\\
        (\sigma_{\cI_j[2]}^2/\beta_{\cI_j[2]} + \Bar{\sigma}_j^2)^{-1}\\
        \vdots\\
        (\sigma_{\cI_j[|\cI_j|]}^2/\beta_{\cI_j[|\cI_j|]} + \Bar{\sigma}_j^2)^{-1}
    \end{bmatrix},
\]
and
\begin{multline*}
    \Omega_{-1}^{-1} = \begin{bmatrix}
        (\Omega_{-1}^{(2)})^{-1} & \bm{0}_{|\cI_2|}\bm{0}_{|\cI_3|}^T & \dots &\bm{0}_{|\cI_2|}\bm{0}_{|\cI_k|}^T\\
        \bm{0}_{|\cI_3|}\bm{0}_{|\cI_2|}^T &(\Omega_{-1}^{(3)})^{-1} & \dots & \bm{0}_{|\cI_3|}\bm{0}_{|\cI_k|}^T\\
        \vdots & \vdots & \vdots & \vdots\\
        \bm{0}_{|\cI_k|}\bm{0}_{|\cI_2|}^T &\bm{0}_{|\cI_k|}\bm{0}_{|\cI_3|}^T &\dots &(\Omega_{-1}^{(k)})^{-1}
    \end{bmatrix} -\\ \left((\Bar{\sigma}^2)^{-1} + \sum_{j = 2}^k \left(\left(\sum_{i' \in \cI_j} \frac{1}{\sigma_i^2/\beta_i + \Bar{\sigma}_j^2}\right)^{-1} + \Bar{\sigma}^2\right)^{-1}\right) 
    \\\times
    \begin{bmatrix}
        \left(\Omega_{-1}^{(2)}\right)^{-1}\mathbf{1}_{|\cI_2|}\\
        \left(\Omega_{-1}^{(3)}\right)^{-1}\mathbf{1}_{|\cI_3|}\\
        \vdots\\
        \left(\Omega_{-1}^{(4)}\right)^{-1}\mathbf{1}_{|\cI_4|}
    \end{bmatrix}        \begin{bmatrix}
        \left(\Omega_{-1}^{(2)}\right)^{-1}\mathbf{1}_{|\cI_2|}\\
        \left(\Omega_{-1}^{(3)}\right)^{-1}\mathbf{1}_{|\cI_3|}\\
        \vdots\\
        \left(\Omega_{-1}^{(4)}\right)^{-1}\mathbf{1}_{|\cI_4|}
    \end{bmatrix}^T.
\end{multline*} 
Similarly, we have
\begin{multline*}
    \mathbf{1}_{m - |\cI_1|}^T\Omega_{-1}^{-1}\mathbf{1}_{m - |\cI_1|} = \left(\Bar{\sigma}^2 + \left[\sum_{j = 2}^k \left(\left(\sum_{i' \in \cI_j} \frac{1}{\sigma_i^2/\beta_i + \Bar{\sigma}_j^2}\right)^{-1} + \Bar{\sigma}^2\right)^{-1}\right]^{-1}\right)^{-1},
\end{multline*}
\begin{multline*}
    \Omega_{-1}^{-1}\mathbf{1}_{m - |\cI_1|} \\
    = 
    \frac{\left[\sum_{j = 2}^k \left(\left(\sum_{i' \in \cI_j} \frac{1}{\sigma_i^2/\beta_i + \Bar{\sigma}_j^2}\right)^{-1} + \Bar{\sigma}^2\right)^{-1}\right]^{-1}}{\Bar{\sigma}^2 + \left[\sum_{j = 2}^k \left(\left(\sum_{i' \in \cI_j} \frac{1}{\sigma_i^2/\beta_i + \Bar{\sigma}_j^2}\right)^{-1} + \Bar{\sigma}^2\right)^{-1}\right]^{-1}}\begin{bmatrix}
        \left(\Omega_{-1}^{(2)}\right)^{-1}\mathbf{1}_{|\cI_2|}\\
        \left(\Omega_{-1}^{(3)}\right)^{-1}\mathbf{1}_{|\cI_3|}\\
        \vdots\\
        \left(\Omega_{-1}^{(4)}\right)^{-1}\mathbf{1}_{|\cI_4|}
    \end{bmatrix},
\end{multline*}
and
\begin{multline*}
    \Omega_1^{-1} = \begin{bmatrix}
        \diag(\{(\sigma_i^2/\beta_i + \Bar{\sigma}_1^2)^{-1}\}_{i = 2}^{|\cI_1|}) & \bm{0}_{|\cI_1| - 1}\bm{0}_{m - |\cI_1|}^T\\
        \bm{0}_{m - |\cI_1|}\bm{0}_{|\cI_1| - 1}^T & \Omega_{-1}^{-1}
    \end{bmatrix} - \\
    \left((\Bar{\sigma}_1^2)^{-1} + \sum_{i = 2}^{|\cI_j|}(\sigma_i^2/\beta_i + \Bar{\sigma}_1^2)^{-1} + \mathbf{1}_{m - |\cI_1|}^T\Omega_{-1}^{-1}\mathbf{1}_{m - |\cI_1|}\right)^{-1} \\
    \times
    \begin{bmatrix}
        (\sigma_2^2/\beta_2 + \Bar{\sigma}_1^2)^{-1}\\
        \vdots\\
        (\sigma_m^2/\beta_2 + \Bar{\sigma}_1^2)^{-1}\\
        \Omega_{-1}^{-1}\mathbf{1}_{m - |\cI_1|}
    \end{bmatrix}\begin{bmatrix}
        (\sigma_2^2/\beta_2 + \Bar{\sigma}_1^2)^{-1}\\
        \vdots\\
        (\sigma_m^2/\beta_2 + \Bar{\sigma}_1^2)^{-1}\\
        \Omega_{-1}^{-1}\mathbf{1}_{m - |\cI_1|}
    \end{bmatrix}^T.
\end{multline*}
Combining the above expressions, we obtain the generalized least squares estimate for $\theta^*_1$, which satisfies
\begin{align*}
    &\left(\beta_1/\sigma_1^2 + \left\{\left[\mathbf{1}_{m - |\cI_1|}^T\Omega_{-1}^{-1}\mathbf{1}_{m - |\cI_1|} + \sum_{i =2}^m \frac{1}{\sigma_i^2/\beta_i + \Bar{\sigma}_1^2}\right]^{-1} + \Bar{\sigma}_i^2\right\}^{-1}\right)\hat{\theta}_1^{GLS} = \\
    &\qquad
        \frac{1}{\sigma_1^2}X_1^Ty_1 + 
        \frac{\left\{\left[\mathbf{1}_{m - |\cI_1|}^T\Omega_{-1}^{-1}\mathbf{1}_{m - |\cI_1|} + \sum_{i =2}^m \frac{1}{\sigma_i^2/\beta_i + \Bar{\sigma}_1^2}\right]^{-1} + \Bar{\sigma}_i^2\right\}^{-1}}{\beta_1/\sigma_1^2 + \left\{\left[\mathbf{1}_{m - |\cI_1|}^T\Omega_{-1}^{-1}\mathbf{1}_{m - |\cI_1|} + \sum_{i =2}^m \frac{1}{\sigma_i^2/\beta_i + \Bar{\sigma}_1^2}\right]^{-1} + \Bar{\sigma}_i^2\right\}^{-1}}
         \\
    &\hspace{6em} \times
        \frac{\sum_{i =2}^m \frac{1}{\sigma_i^2/\beta_i + \Bar{\sigma}_1^2}\hat{\theta}_i^d}{\mathbf{1}_{m - |\cI_1|}^T\Omega_{-1}^{-1}\mathbf{1}_{m - |\cI_1|} + \sum_{i =2}^m \frac{1}{\sigma_i^2/\beta_i + \Bar{\sigma}_1^2}}\\
    &\qquad +
        \frac{\left\{\left[\mathbf{1}_{m - |\cI_1|}^T\Omega_{-1}^{-1}\mathbf{1}_{m - |\cI_1|} + \sum_{i =2}^m \frac{1}{\sigma_i^2/\beta_i + \Bar{\sigma}_1^2}\right]^{-1} + \Bar{\sigma}_i^2\right\}^{-1}}{\beta_1/\sigma_1^2 + \left\{\left[\mathbf{1}_{m - |\cI_1|}^T\Omega_{-1}^{-1}\mathbf{1}_{m - |\cI_1|} + \sum_{i =2}^m \frac{1}{\sigma_i^2/\beta_i + \Bar{\sigma}_1^2}\right]^{-1} + \Bar{\sigma}_i^2\right\}^{-1}} \\
    &\hspace{5em} \times
        \frac{\left[\sum_{j = 2}^k \left(\left(\sum_{i' \in \cI_j} \frac{1}{\sigma_i^2/\beta_i + \Bar{\sigma}_j^2}\right)^{-1} + \Bar{\sigma}^2\right)^{-1}\right]^{-1}}{\mathbf{1}_{m - |\cI_1|}^T\Omega_{-1}^{-1}\mathbf{1}_{m - |\cI_1|} + \sum_{i =2}^m \frac{1}{\sigma_i^2/\beta_i + \Bar{\sigma}_1^2}}\\
    &\hspace{5em} \times \sum_{j = 2}^k \frac{\left(\sum_{i' \in \cI_j} \frac{1}{\sigma_i^2/\beta_i + \Bar{\sigma}_j^2}\right)^{-1}\left(\sum_{i' \in \cI_j} \frac{1}{\sigma_i^2/\beta_i + \Bar{\sigma}_j^2}\hat{\theta}_{i'}^d\right)}{\left(\sum_{i' \in \cI_j} \frac{1}{\sigma_i^2/\beta_i + \Bar{\sigma}_j^2}\right)^{-1} + \Bar{\sigma}^2}.
\end{align*}
Observe that $\hat{\theta}_1^{GLS}$ is exactly $\hat{\theta}_1$ when $\gamma_i = \frac{1}{\Bar{\sigma}_j^2}$ for all $i \in \cI_j, j = 1, \dots, k$ and $\lambda_j = \frac{1}{\Bar{\sigma}^2}$. Our claim holds by the Gauss-Markov theorem \citep{kariya2004generalized}. 

\subsection{Proof of \cref{prop:not_gls}} 
\label{subsec:proof_of_prop_not_gls}

Under our model, for $i \neq 1$, $i \in \cI_j$,
\[
        y_i = X_i\theta_1^* - X_i(\xi_1 + \Bar{\xi}_1 - \Bar{\xi}_j - \xi_i) + \epsilon_i.
\] 
Therefore, similar to \cref{eqn:y_theta_distribution}, we have
\[
        \begin{bmatrix}
            y_1 \\
            \hat{\theta}_2^d \\
            \vdots\\
            \hat{\theta}_m^d
        \end{bmatrix} = \begin{bmatrix}
            X_1\\
            X_2\\
            \vdots\\
            X_n
        \end{bmatrix}\theta_1^* +\begin{bmatrix}
            \epsilon_1\\
            - X_2(\xi_1 - \xi_2) + \epsilon_2\\
            \vdots\\
            - X_n(\xi_1 + \Bar{\xi}_1 - \Bar{\xi}_k - \xi_n) + \epsilon_n.
        \end{bmatrix}.
\]
Let $\Tilde{\theta}_1$ be the solution to the generalized least squares problem defined by the equation above. In general, 
$\Tilde{\theta}_1 \neq \hat{\theta}_1$. For example, when $X_i \neq I$ for all $i$. Since $\hat{\theta}_i^d$ are linear in $(X_i, y_i)$, $\hat{\theta}_1$ is linear in $\{(X_i, y_i)\}_{i = 1}^n$. Our proposition then holds by the Gauss-Markov theorem \citep{kariya2004generalized}.

\subsection{Proof of \cref{prop:JS}}
\label{subsec:proof_of_prop_JS}

Similarly to \cref{eqn:y_theta_distribution}, we have
\[
        \hat{\theta}_1 = \frac{n - 1}{2n}y_1 + \frac{n + 1}{2n}\Bar{y}_{-1},
\] 
where $$\Bar{y}_{-1} = \frac{1}{n - 1}\sum_{i = 2}^n y1.$$ Then
\[
        \EE[\|\hat{\theta}_1 - \theta_1^*\|^2] = \frac{(n - 1)^2}{4n^2}d + d + \frac{(n + 1)^2}{4n^2}\EE[\|\Bar{y}_{-1} - \Bar{\theta}^*\|^2].
\]
The estimator $\Tilde{\theta}^{JS}_1$ is obtained as
\[
        \Tilde{\theta}_1^{JS} = \frac{n - 1}{2n}y_1 + \frac{n + 1}{2n} C\Bar{y}_{-1},
\] 
where $C \in [0, 1]$ is a shrinkage parameter. Then
\[
    \EE[\|\Tilde{\theta}_1^{JS} - \theta_1^*\|^2] = \frac{(n - 1)^2}{4n^2}d + d + \frac{(n + 1)^2}{4n^2}\EE[\|C\Bar{y}_{-1} - \Bar{\theta}^*\|^2].
\]
When estimating $\Bar{\theta}^*$ and $d > 3$, there is a $C < 1$ such that
\[
        \EE[\|C\Bar{y}_{-1} - \Bar{\theta}^*\|^2] \leq \EE[\|\Bar{y}_{-1} - \Bar{\theta}^*\|^2],
\] 
where the exact form of C is discussed in \cite{james1992estimation, bock1975minimax, kubokawa1991approach}. Consequently, 
\[
        \EE[\|\Tilde{\theta}_1^{JS} - \theta_1^*\|^2] \leq \EE[\|\hat{\theta}_1 - \theta_1^*\|^2].
\] 
Since
\[
        \EE[\Tilde{\theta}_1^{JS} \mid \theta_1^*] = \rbr{\frac{n - 1}{2n} + \frac{C(n + 1)}{2n} }\theta_1^*= \rbr{1 - \frac{n + 1}{2n}(1 - C)}\theta_1^*
\] and
$C < 1$, $\Tilde{\theta}_1^{JS}$ is biased.

\section{Useful Results}
\begin{theorem}
\label{thm:breg_smooth}

Let $g(x): \RR^d \to \RR$ be an $L_g$-smooth and convex function. Let
\[
    D_g(x, y) = g(x) - g(y) - (\nabla g(y))^T(x - y), \quad x, y \in \RR^d.
\] 
Then, for all $x, y \in \RR^d$, we have
\[
    \|\nabla g(x) - \nabla g(y)\|_2^2 \leq 2L_g D_g(x, y).
\]
\end{theorem}
\begin{proof}
    Directly follows from (2.1.10) in Theorem 2.1.5 of \cite{nesterov2018lectures}.
\end{proof}

\end{document}